%% file: main.tex
\documentclass[nohyperref]{article}

\usepackage{microtype}
\usepackage{graphicx}
\usepackage{subcaption}
\usepackage{booktabs} % for professional tables

\usepackage{hyperref}

\usepackage[accepted]{icml2022}

\usepackage{amsmath}
\usepackage{amssymb}
\usepackage{mathtools}
\usepackage{amsthm}

\usepackage[capitalize,noabbrev]{cleveref}

\theoremstyle{plain}
\newtheorem{theorem}{Theorem}[section]
\newtheorem{proposition}[theorem]{Proposition}
\newtheorem{lemma}[theorem]{Lemma}

\theoremstyle{definition}

\theoremstyle{remark}

\newtheorem*{theorem*}{Theorem}
\usepackage{pifont}

\usepackage{tikz}
\newcommand{\circled}[1]{\textcircled{\raisebox{-0.9pt}{#1}}}
\usepackage[textsize=tiny]{todonotes}
\newcommand{\red}[1]{\textcolor{red}{#1}}

\input{math_command}

\icmltitlerunning{Model-based Stationary Distribution Regularization}

\begin{document}

\twocolumn[
\icmltitle{
%Stabilizing Offline Reinforcement Learning \\ via Model-Based Stationary Distribution Regularization\\
%Modeling and Regularizing Policy Stationary  Distribution to Stabilize Offline Reinforcement Learning \\
 Regularizing a Model-based Policy Stationary  Distribution\\ to Stabilize Offline Reinforcement Learning
}

% It is OKAY to include author information, even for blind
% submissions: the style file will automatically remove it for you
% unless you've provided the [accepted] option to the icml2022
% package.

% List of affiliations: The first argument should be a (short)
% identifier you will use later to specify author affiliations
% Academic affiliations should list Department, University, City, Region, Country
% Industry affiliations should list Company, City, Region, Country

% You can specify symbols, otherwise they are numbered in order.
% Ideally, you should not use this facility. Affiliations will be numbered
% in order of appearance and this is the preferred way.
\icmlsetsymbol{equal}{*}

\begin{icmlauthorlist}
\icmlauthor{Shentao Yang}{utmccombs}
\icmlauthor{Yihao Feng}{utcs}
\icmlauthor{Shujian Zhang}{utsds}
\icmlauthor{Mingyuan Zhou}{utmccombs,utsds}
% \icmlauthor{Firstname5 Lastname5}{yyy}
% \icmlauthor{Firstname6 Lastname6}{sch,yyy,comp}
% \icmlauthor{Firstname7 Lastname7}{comp}
% %\icmlauthor{}{sch}
% \icmlauthor{Firstname8 Lastname8}{sch}
% \icmlauthor{Firstname8 Lastname8}{yyy,comp}
%\icmlauthor{}{sch}
%\icmlauthor{}{sch}
\end{icmlauthorlist}

\icmlaffiliation{utmccombs}{McCombs School of Business,}
\icmlaffiliation{utsds}{Department of Statistics \& Data Science, The University of Texas at Austin}
\icmlaffiliation{utcs}{Department of Computer Science,}

\icmlcorrespondingauthor{Shentao Yang}{shentao.yang@mccombs.utexas.edu}

\icmlcorrespondingauthor{Mingyuan Zhou}{mingyuan.zhou@mccombs.utexas.edu}

% You may provide any keywords that you
% find helpful for describing your paper; these are used to populate
% the "keywords" metadata in the PDF but will not be shown in the document
\icmlkeywords{Machine Learning, ICML}

\vskip 0.3in
]

% this must go after the closing bracket ] following \twocolumn[ ...

% This command actually creates the footnote in the first column
% listing the affiliations and the copyright notice.
% The command takes one argument, which is text to display at the start of the footnote.
% The \icmlEqualContribution command is standard text for equal contribution.
% Remove it (just {}) if you do not need this facility.

\printAffiliationsAndNotice{}  % leave blank if no need to mention equal contribution
% \printAffiliationsAndNotice{\icmlEqualContribution} % otherwise use the standard text.

\begin{abstract}
Offline reinforcement learning (RL) extends the paradigm of classical RL algorithms to purely learning from static datasets, without interacting with the underlying environment during the learning process. A key challenge of offline RL is the instability of policy training, caused by the mismatch between the distribution of the offline data and the undiscounted stationary state-action distribution of the learned policy. To avoid the detrimental impact of distribution mismatch, we regularize the undiscounted stationary distribution of the current policy towards the offline data during the policy optimization process. Further, we train a dynamics model to both implement this regularization and better estimate the stationary distribution of the current policy, reducing the error induced by distribution mismatch. On a wide range of continuous-control offline RL datasets, our method indicates competitive performance, which validates our algorithm. The code is publicly \href{https://github.com/Shentao-YANG/SDM-GAN_ICML2022}{available}.
\end{abstract}
\input{Tex/intro}

\input{Tex/background}

\input{Tex/main}
\input{Tex/related}
\input{Tex/exp}
\input{Tex/conclusion}

% Note use of \abovespace and \belowspace to get reasonable spacing
% above and below tabular lines.

% In the unusual situation where you want a paper to appear in the
% references without citing it in the main text, use \nocite
% \clearpage
\bibliography{reference}
\bibliographystyle{icml2022}

%%%%%%%%%%%%%%%%%%%%%%%%%%%%%%%%%%%%%%%%%%%%%%%%%%%%%%%%%%%%%%%%%%%%%%%%%%%%%%%
%%%%%%%%%%%%%%%%%%%%%%%%%%%%%%%%%%%%%%%%%%%%%%%%%%%%%%%%%%%%%%%%%%%%%%%%%%%%%%%
% APPENDIX
%%%%%%%%%%%%%%%%%%%%%%%%%%%%%%%%%%%%%%%%%%%%%%%%%%%%%%%%%%%%%%%%%%%%%%%%%%%%%%%
%%%%%%%%%%%%%%%%%%%%%%%%%%%%%%%%%%%%%%%%%%%%%%%%%%%%%%%%%%%%%%%%%%%%%%%%%%%%%%%
\clearpage
\appendix
\onecolumn

\begin{center}
\Large
\textbf{Appendix}
\end{center}

\input{Tex/appendix}

%%%%%%%%%%%%%%%%%%%%%%%%%%%%%%%%%%%%%%%%%%%%%%%%%%%%%%%%%%%%%%%%%%%%%%%%%%%%%%%
%%%%%%%%%%%%%%%%%%%%%%%%%%%%%%%%%%%%%%%%%%%%%%%%%%%%%%%%%%%%%%%%%%%%%%%%%%%%%%%

\end{document}

%% file: math_command.tex
\makeatletter
\newtheorem*{rep@theorem}{\rep@title}
\newcommand{\newreptheorem}[2]{%
\newenvironment{rep#1}[1]{%
 \def\rep@title{#2 \ref{##1}}%
 \begin{rep@theorem}}%
 {\end{rep@theorem}}}
\makeatother

\usepackage{amsmath,amsfonts,bm,amsthm,amssymb,mathrsfs,algorithm,algorithmic}
\usepackage{hyperref,url,booktabs,enumitem,multicol,graphicx,caption,subcaption}

\renewcommand{\algorithmicrequire}{\textbf{Input:}}
\renewcommand{\algorithmicensure}{\textbf{Output:}}

\def\Secref#1{Section~\ref{#1}}
\def\eqref#1{(\ref{#1})}
\def\1{\bm{1}}

\def\vzero{{\bm{0}}}

\def\vtheta{{\bm{\theta}}}
\def\va{{\bm{a}}}

\def\vr{{\bm{r}}}
\def\vs{{\bm{s}}}

\def\vw{{\bm{w}}}

\def\vphi{{\bm{\phi}}}
\def\vepsilon{{\bm{\epsilon}}}

\def\mI{{\bm{I}}}

\DeclareMathAlphabet{\mathsfit}{\encodingdefault}{\sfdefault}{m}{sl}
\SetMathAlphabet{\mathsfit}{bold}{\encodingdefault}{\sfdefault}{bx}{n}

\def\gB{{\mathcal{B}}}

\def\gD{{\mathcal{D}}}
\def\gE{{\mathcal{E}}}
\def\gF{{\mathcal{F}}}
\def\gG{{\mathcal{G}}}

\def\gL{{\mathcal{L}}}
\def\gM{{\mathcal{M}}}
\def\gN{{\mathcal{N}}}

\def\gP{{\mathcal{P}}}

\def\sA{{\mathbb{A}}}

\def\sD{{\mathbb{D}}}
\def\sS{{\mathbb{S}}}

\def\cG{{\mathcal{G}}}
\def\cF{{\mathcal{F}}}

\newcommand{\E}{\mathbb{E}}

\newcommand{\abs}[1]{\left\vert#1\right\rvert}
\newcommand{\norm}[1]{\left\Vert#1\right\Vert}
\newcommand{\cbr}[1]{\left\{#1\right\}}
\newcommand{\br}[1]{\left(#1\right)}
\newcommand{\sbr}[1]{\left[#1\right]}
\newcommand{\given}{\,|\,}

\newcommand{\denv}{\gD_{\mathrm{env}}}
\newcommand{\dmodel}{\gD_{\mathrm{model}}}

\newcommand{\dbtrue}{d_{\pi_b}^{P^{*}}}
\newcommand{\dpitrue}{d_{\pi}^{P^{*}}}

\newcommand{\wQ}{{\widehat{Q}}}
\newcommand{\wP}{{\widehat{P}}}
\newcommand{\trueP}{{P^{*}}}

\usepackage{relsize}
\newcommand*{\defeq}{\stackrel{\mathsmaller{\mathsf{def}}}{=}}

%% file: Tex/intro.tex
\section{Introduction}\label{sec:intro}

% General benefits of offline rl 
Offline reinforcement learning (RL), traditionally known as batch RL, eschews environmental interactions during the policy learning process and focuses on training policy from static dataset collected by one or more data-collecting policies, which we collectively call the behavior policy \citep{treerl2005, batchrl2012, offlinetutorial2020}.
This paradigm extends the applicability of RL from trial-and-error in the simulator \cite{dqn2013,dqn2015,alphago2016,alphazero2017} to applications where environmental interactions can be costly or even dangerous \citep{whentotreat2019,autodrivesurvey2020}, and to domains with abundant accumulated data \citep{offlineabtest2018}.
% Drawbacks of offline RL: distribution-shift 
While classical off-policy RL algorithms \citep{dqn2013,ddpg2016,c512017,qrdqn2017,td32018,sacnew2018} can be directly applied to the offline setting, they often fail to learn purely from a static dataset, especially when the offline dataset exhibits a narrow distribution on the state-action space \citep{bcq2019,bear2019}.
One important cause to such a failure is the mismatch between the stationary state-action distribution of the behavior policy and that of the learned policy, resulting in possibly pathological estimates of the queried action-values during the training process and thus the instability and failure of the policy learning.

To successfully learn a policy from the offline dataset, prior model-free offline RL works try to avoid the action-distribution shift during the training process, so that the overestimated action-values in out-of-distribution (OOD) actions can be mitigated.
Their approaches can be mainly characterized into three categories: 
\textbf{(1)} robustly train the action-value function or provide a conservative estimate of the action-value function whose numerical values distinguish in- and out-of-distribution actions \citep{rem2020,cql2020,addressingextra2021,s4rl2021};
\textbf{(2)} design a tactful behavior-cloning scheme to learn only from ``good" actions in the dataset \citep{crr2020,decisiontrans2021,iql2021};
\textbf{(3)} regularize the current policy to be close to the action choice in the offline dataset during the training process, so that the learned policy mainly takes action within the reach of the action-value function estimate \citep{bcq2019,sibb2019,bear2019,brac2019,offlinerldialog2019,abm2020,oraac2021,fisherbrc2021,uwac2021,td3bc2021}.

Model-based RL (MBRL) can provide a potential enhancement over the model-free RL.
By learning an approximate dynamic model and subsequently using that for policy learning, MBRL can provide a sample-efficient solution to answer the counterfactual question pertaining to action-value estimation in the online setting \citep{mbpo2019,gamembrl2020}, and to provide an augmentation to the static dataset in the offline setting \citep{mopo2020}.
% model based method: still distribution shift so need to constraint the state-action visitation of the learned policy
Directly applying model-based RL methods into the offline setting, however, still faces the challenge of mismatching stationary state-action distribution.
In particular, when the offline dataset only covers a narrow region on the state-action space, the estimated dynamic model may not support far extrapolation \citep{wmopo2021}.
Consequently, poor empirical performance related to the ``model exploitation" issue can appear when directly using the learned model for offline policy learning \citep{asi2012,mbmpo2018,metrpo2018,gamembrl2020,mopo2020}.
As remedies, prior model-based offline RL methods try to avoid state-action pairs far from the offline dataset by uncertainty quantification \citep{mopo2020, morel2020, fan2021contextual,moose2021}, model-based constrained policy optimization \citep{bremen2021,mabe2021}, or model-based conservative action-value function estimation \citep{combo2021}.

% regularize the stationary distribution of the current policy
In this work, we propose to directly regularize the undiscounted stationary state-action distribution of the current policy towards that of the behavior policy during the policy learning process.
The learned policy thus avoids visiting state-action pairs far from the offline dataset, reducing the occurrence of overestimated action-values and stabilizing policy training.
%  train a dynamic model
{
We derive a tractable bound for the distance between the undiscounted stationary distributions.
To implement this bound, we \textbf{(1)} train a dynamic model in a sufficiently rich function class via the maximum likelihood estimation (MLE);
and \textbf{(2)} add a tractable regularizer into the policy optimization objective. This regularizer only requires samples from the offline dataset, the short-horizon model rollouts, the current policy, and the estimated dynamic.}
Further, model-based synthetic rollouts are used to better estimate the stationary state-action distribution of the current policy, which helps the performance of the final policy. 
%the statistical error induced by mismatched stationary distributions and facilitates generalization beyond the static dataset.
% implicit policy
Besides, our framework allows training an implicit policy to better approximate the maximizer of the action-value function estimate, particularly when the latter exhibits multi-modality.
% competitive performance
Without assuming any knowledge about the underlying environment, including the reward function and the termination condition, our method shows competitive performance on a wide range of continuous-control offline-RL datasets from the D4RL benchmark \citep{fu2021d4rl}, validating the effectiveness of our algorithmic designs.

%% file: Tex/background.tex
\section{Background} \label{sec:background}

% RL
We follow the classical RL setting \citep{rlintro2018} to model the interaction between the agent and the environment as a Markov Decision Process (MDP), specified by the tuple $\gM = \br{\sS, \sA, P, r, \gamma, \mu_0}$, where $\sS$ denotes the state space, $\sA$ the action space, $P\br{s' \given s, a}: \sS \times \sS \times \sA \rightarrow [0,1]$ the environmental dynamic, $r(s,a): \sS \times \sA \rightarrow [-r_{max}, r_{max}]$ the reward function, $\gamma \in (0,1]$ the discount factor, and $\mu_0(s): \sS \rightarrow [0,1]$ the initial-state distribution.

% OPE (def of both discounted and undiscounted distribution)
For a given policy $\pi(a\given s)$, we follow the literature \citep{markovdp1994,breakingcurse2018}
to denote the state-action distribution at timestep $t \geq 0$ induced by policy $\pi$ on MDP $\gM$ as 
$
 d_{\pi, t}(s,a) :=  \Pr\br{s_t = s, a_t = a \given s_0 \sim \mu_0, a_t \sim \pi, s_{t+1} \sim P,\forall\, t \geq 0}.
$
Denote the \textit{discounted} stationary state-action distribution induced by $\pi$ as
$
    d_{\pi,\gamma}(s,a) = (1 - \gamma) \sum_{t=0}^\infty \gamma^t d_{\pi, t}(s,a)
$
and the \textit{undiscounted (average)} stationary state-action distribution as $d_\pi(s,a) = \lim_{T \rightarrow \infty}\sum_{t=0}^T \frac{1}{T+1} d_{\pi, t}(s,a)$.
We have $d_{\pi,\gamma}(s,a) = d_{\pi,\gamma}(s) \pi(a \given s)$ and $d_\pi(s,a) = d_\pi(s) \pi(a \given s)$.
% offline RL background
In offline RL \citep{offlinetutorial2020}, one has only access to a static dataset $\denv = \cbr{(s,a,r,s')}$ collected by some behavior policy $\pi_b$, which can be a mixture of several data-collecting policies.

% Actor-Critic Algorithm (def of Q, training of Q, training of \pi)
Denote the action-value function for policy $\pi$ on $\gM$ as $Q^\pi(s,a) = \E_{\pi, P, r}\sbr{\sum_{t=0}^\infty \gamma^t r(s_t,a_t) \given s_0 = s, a_0 = a}$.
In the actor-critic algorithm \citep{rlintro2018}, the policy $\pi$ and critic $Q^\pi(s,a)$ are typically modelled as parametric functions $\pi_\vphi$ and $Q_\vtheta$,  parametrized by $\vphi$ and $\vtheta$, respectively.
In offline RL, the critic is trained in the policy evaluation step by Bellman backup as minimizing {\it w.r.t.} $\vtheta$:
% \begin{equation}\textstyle \label{eq:ac_bellman}
% \resizebox{.48\textwidth}{!}{%
% $
%     \E_{(s,a,s')\sim \denv} \sbr{\br{Q_{\vtheta}(s, a) - \br{ r(s, a) + \gamma \E_{\, a' \sim \pi(\cdot \given s')} \sbr{Q_{\vtheta}(s', a')}}}^2};
% $%    
% }
% \end{equation}
\begin{align}
    &\ell(\vtheta):= \E_{(s,a,r,s')\sim \denv} \left[\left(Q_{\vtheta}(s,a)  - \hat{B}^\pi Q_{\vtheta^\prime}(s,a)\right)^2\right]\,,\notag \\ 
    &\hat{B}^\pi Q_{\vtheta^\prime}(s,a)= r(s,a) + \gamma \E_{a^\prime\sim \pi(\cdot  \given s^\prime)}\left[Q_{\vtheta^\prime}(s^\prime, a^\prime)\right]\,, \label{eq:ac_bellman}
\end{align}
where $Q_{\vtheta^\prime}$ is the target network. 
The actor $\pi_\vphi$ is trained in the policy improvement step as
\begin{equation}\textstyle \label{eq:ac_actor}
     \arg\max_\vphi \E_{s \sim d_{\pi_b}(s),\, a \sim \pi_{\vphi}(\cdot \given s)}\sbr{Q_{\vtheta}\br{s, a}},
\end{equation}
where samples from the offline dataset $\denv$ is used to approximate the samples from $d_{\pi_b}(s)$ \citep{diagbottleneck2019, offlinetutorial2020}.

% model-based RL
In model-based offline RL, a stochastic dynamic model $\widehat P\br{s' \given s, a}$ within some function class $\mathcal{P}$ is learned to approximate the true environmental dynamic, denoted as $P^*$.
With the offline dataset $\denv$, $\widehat P$ is typically trained using MLE \citep{mbpo2019, mopo2020, combo2021}  as
\begin{equation}\textstyle \label{eq:model_objective}
    \arg\max_{\widehat P \in \mathcal{P}} \E_{\br{s,a,s'} \sim \denv}\sbr{\log \widehat P\br{s' \given s, a}}.
\end{equation}
Models for the reward function $\hat r$ and the termination conditional can also be trained similarly if assumed unknown.
With $\widehat P$ and $\hat r$, an approximated MDP to $\gM$ can be constructed as $\widehat \gM = (\sS, \sA, \widehat P, \hat r, \gamma, \mu_0)$.
To distinguish the aforementioned stationary distributions induced by policy $\pi$ on $\gM$ and on $\widehat \gM$, we denote the undiscounted stationary state-action distribution induced by $\pi$ on the true dynamic $P^*$ (or $\gM$) as $d_\pi^{P^*}(s,a)$, on the learned dynamic $\widehat P$ (or $\widehat M$) as $d_\pi^{\widehat P}(s,a)$, and similarly for the discounted stationary state-action distributions $d_{\pi,\gamma}^{P^*}(s,a)$ and $d_{\pi,\gamma}^{\widehat P}(s,a)$.
With the estimated dynamic $\widehat P$, prior model-based offline RL works \citep{morel2020,mopo2020,combo2021,mabe2021} typically approximate $d_{\pi_\vphi}^{P^*}(s,a)$ by simulating the learned policy $\pi_\vphi$ in $\widehat M$ for a short horizon $h$ starting from state $s \in \denv$.
The resulting trajectories are stored into a replay buffer $\dmodel$, similar to the replay buffer in off-policy RL \citep{replaybuffer1992,ddpg2016}.
Sampling from $\denv$ in Eqs.~\eqref{eq:ac_bellman} and \eqref{eq:ac_actor} is commonly replaced by sampling from the augmented dataset $\mathcal{D} := f \denv + (1-f) \dmodel, f \in [0,1]$, denoting  sampling from $\denv$ with probability $f$ and from $\dmodel$ with probability $1 - f$.

We follow the offline RL literature \citep{breakingcurse2018, algaedice2019, confoundingrobust2020,gendice2020} to assume
\textbf{(i)} that all the Markov chains induced by the studied (approximated) dynamics and policies are ergodic, and thus the undiscounted stationary state-action distribution $d_\pi(s,a)$ equals to the limiting state-action occupancy measure induced by $\pi$ on the corresponding MDP;
and \textbf{(ii)} that the offline dataset $\denv$ is constructed as the rollouts of $\pi_b$ on the true dynamic $P^*$, \textit{i.e.}, $\denv \sim d_{\pi_b}^{P^*}(s,a)$.
We denote the state-action and state distributions in $\denv$ as $d_{\denv}(s,a)$ and $d_{\denv}(s)$, which are  discrete approximations to $d_{\pi_b}^{P ^*}(s,a)$ and $d_{\pi_b}^{P ^*}(s)$, respectively.

%% file: Tex/main.tex
\section{Main Method} \label{sec:main_method}

In this section, we introduce our approach to stabilize the policy training. 
Specifically, except optimizing the policy $\pi$ to maximize the action-value function, we add a distribution regularization into the policy optimization objective, which encourages the closeness between the stationary distribution of the learned policy $\pi$ and that of the behavior $\pi_b$.
Our policy optimization objective is summarized as
\begin{align*}
    J(\pi):= \lambda \E_{s \sim \denv,\, a \sim \pi(\cdot | s)}\sbr{Q_{\vtheta}\br{s, a}} -  D(\dbtrue, \dpitrue),
\end{align*}
where $\lambda$ is the regularization coefficient, $D(\cdot, \cdot)$ is a statistical distance  between two probability distributions, such as the integral probability metric (IPM, \citet{ipm1997}) and the Jensen--Shannon divergence (JSD, \citet{jsd1991}),
$\dbtrue(s,a)$ is the stationary distribution induced by behavior policy $\pi_{b}$ in the true dynamic $P^{*}$, and $\dpitrue(s,a)$ is the stationary distribution induced by the current policy $\pi$ in $P^{*}$.
 
\subsection{A Tractable Bound to the Regularization Term} \label{sec:method_theory}
As discussed in \Secref{sec:background}, the offline dataset $\denv$ is drawn from $\dbtrue$. 
Since the current policy $\pi$ can not interact with the environment during the training process, we can not directly estimate $\dpitrue$ using simulation. 
Hence, we use a learned dynamic model $\widehat P$ to approximate $\dpitrue$. 
Using the triangle inequality for the distance $D$, we have
\begin{align}
    D(\dbtrue, \dpitrue) \leq D(\dbtrue, d_{\pi}^{\widehat P}) + D(d_{\pi}^{\widehat P}, d_{\pi}^{P^*})\,. \label{equ:reg_relax}
\end{align}
Next, we describe how to estimate $D(\dbtrue, d_{\pi}^{\widehat P})$ and how to upper bound $D(d_{\pi}^{\widehat P}, d_{\pi}^{P^*})$ in the RHS of Eq.~\eqref{equ:reg_relax}. % separately. 

{
For a given dynamic model $\widehat P(s^\prime | s,a)$, since we do not know the formula of $d_{\pi}^{\widehat{P}}(s,a)$, a natural way to estimate this distribution is collecting rollouts of $\pi$ on the dynamic $\widehat P$. 
A major drawback of this approach is that it requires drawing many trajectories. 
This can be time-consuming as trajectories are typically long or even infinite, especially when they converge to the stationary distribution slowly. 
Even worse, since $\pi$ changes during the policy-learning process, we have to repeat this expensive rollout process many times.
To estimate $D(\dbtrue, d_{\pi}^{\widehat P})$ in Eq.~\eqref{equ:reg_relax}, we instead derive the following theorem, which avoids sampling full trajectories and only requires sampling a \textit{single} $(s^\prime, a^\prime)$ pair using policy $\pi$ and model $\widehat{P}$ starting from $(s,a)$ drawn from the offline dataset. 
This can significantly save computation and training time.
}

\begin{theorem}\label{prop:ipm1}
If the distance metric $D$ is the IPM w.r.t some function class $\mathcal{G}=\{g: \sS \times \sA \to \mathbb{R}\}$, defined as 
\begin{align*}
    D_{\cG}(\dbtrue, d_{\pi}^{\widehat P}):= \sup_{g\in \mathcal{G}}\left\vert\E_{\dbtrue}\left[g(s,a)\right] - \E_{d_{\pi}^{\widehat P}}\left[g(s,a)\right]\right\vert\,,
\end{align*}
then there exists a $\gG$ s.t. we can rewrite $D_{\cG}(\dbtrue, d_{\pi}^{\widehat P})$ as:
\begin{equation*} 
    \resizebox{.48\textwidth}{!}{%
    $
\begin{aligned}
    &D_{\cG}(\dbtrue, d_{\pi}^{\widehat P})= \mathcal{R}_{\cF}(\dbtrue, \pi,\widehat{P}) \notag \\
    :=&\sup_{f\in \cF} \bigg\vert\E_{(s,a)\sim \dbtrue  }[f(s,a)]-\E_{{s^\prime \sim \widehat P(\cdot | s,a), a^\prime \sim \pi(\cdot | s^\prime)}\atop {(s,a)\sim \dbtrue }}\left[f(s^\prime, a^\prime)\right]\bigg\vert\,, 
\end{aligned}
$%
}
\end{equation*}
where 
\begin{equation*}
    \resizebox{.48\textwidth}{!}{%
$
\begin{aligned}
         \cF:=\Bigg\{f:&~ f(s,a)=\E_{\pi, \widehat{P}}\left[\sum_{t=0}^{\infty}\left(g(s_t,a_t) - \widehat{\eta}^\pi\right)~|~{{a_0=a}\atop{s_0=s}}\right], \, \\
    &\widehat{\eta}^\pi =\lim_{T\to \infty}\E_{\pi, \widehat{P}}\left[\frac{1}{T+1}\sum_{t=0}^{T}g(s_t, a_t)\right], g\in \cG\Bigg\}\,. 
    \end{aligned}
    $%    
}
\end{equation*}
\end{theorem}

We describe the definition of the function class $\gG$ in Appendix~\ref{sec:proof_31}.
In Theorem \ref{prop:ipm1}, the supremum is taken over a new function class $\mathcal{F}$, which captures the \textit{relative value function} following the policy $\pi$ and the transition $\widehat{P}$ when the reward function is the deterministic function $g(s,a)$ \citep{markovdp1994,rlintro2018}. 
{Using this new function class $\gF$ instead of the original $\gG$ enables the aforementioned avoidance of full-trajectory sampling.}
A more detailed discussion on the relation between $f(s,a)$ and $g(s,a)$ can be found in  Appendix \ref{sec:proof_31}.

{
We remark that in practice the relative value function $f$ is approximated by neural networks with regularity, under the same assumptions as approximating the solution of Bellman backup using neural networks.
Because of the expressiveness of neural networks, we assume that $f$ then resides in a richer function class that contains the class $\mathcal{F}$ of differential value functions. 
Since $f$ is maximized over a larger function class than $\gF$, we effectively minimize an upper bound of the original regularization term in Theorem~\ref{prop:ipm1}.
This relaxation avoids explicitly estimating the differential value function for each reward function $g$ encountered in the training process. Such explicit estimation can be time-consuming.
}

Our next step is to bound $D(d_{\pi}^{\widehat P}, d_{\pi}^{P^*})$. 

\begin{theorem}\label{thm:model_error}
Let $P^{*}$ be the true dynamic, $D_\gG$ the IPM with the same $\gG$ as Theorem~\ref{prop:ipm1}, and $\widehat{P}$ the estimated dynamic.
There exist a constant $C \geq 0$ and a function class $\Psi =\{\psi: \sS \times \sA \to \mathbb{R}\}$ such that for any policy $\pi$,
\begin{equation} \label{equ:mb_upper}
    \resizebox{.48\textwidth}{!}{%
$
\begin{aligned}
    D_{\cG}(d_{\pi}^{\widehat P}, d_{\pi}^{P^*}) \leq & C\cdot \E_{(s,a)\sim \dbtrue}\left[\sqrt{\frac{1}{2}\mathrm{KL}\left(P^{*}(s^\prime |s,a) || \widehat{P}(s^\prime | s,a) \right)}\right]   \\
    & \quad +  \mathcal{R}_{\Psi}(\dbtrue, \pi)\,, 
\end{aligned}
    $%    
}
\end{equation}
where 
\begin{equation*}
    \resizebox{.48\textwidth}{!}{%
$
\begin{aligned}
 \mathcal{R}_{\Psi}(\dbtrue, \pi):=\sup_{\psi \in \Psi}\left\vert \E_{(s,a)\sim \dbtrue}[\psi(s,a)] - \E_{s\sim \dbtrue\atop{a\sim \pi(\cdot|s)}}[\psi(s,a)]\right\vert\,.
\end{aligned}
    $%    
}
\end{equation*}
%\yf{TODO: check the description of the assumption}
\end{theorem}
We describe the definition of the function class $\Psi$ in Appendix~\ref{sec:proof_thm_32}. 
Detailed proof of Theorem \ref{thm:model_error} can be found in Appendix \ref{sec:proof_thm_32}. 

\begin{proposition}\label{prop:kl_mle_sec_method}
An upper bound of the first term on the right-hand-side of Eq.~\eqref{equ:mb_upper}
can be minimized via the MLE for $\widehat P$, {\it i.e.}, Eq.~\eqref{eq:model_objective}.
\end{proposition}
Informally, this proposition coincides with the intuition that estimating the environmental dynamic using MLE can be helpful for matching $d_{\pi}^{\widehat P}$ with $d_{\pi}^{P^*}$.
This proposition is restated and proved in Proposition \ref{eq:kl_mle} in Appendix~\ref{sec:proof_thm_32}.

\begin{theorem}\label{col:sum_up}
Combine Theorems \ref{prop:ipm1} and \ref{thm:model_error} and Proposition \ref{prop:kl_mle_sec_method}, and suppose the dynamic model $\widehat P$ is trained by the MLE objective (Eq.~\eqref{eq:model_objective}), then for the same $\gG$ and any policy $\pi$, up to some constants {\it w.r.t.} $\pi$ and $\widehat P$, we have
\begin{equation*}
    \resizebox{.48\textwidth}{!}{%
$
\begin{aligned}
    D_{\cG}(\dbtrue, \dpitrue) \leq \mathcal{R}_{\cF}(\dbtrue, \pi,\widehat{P}) + \mathcal{R}_{\Psi}(\dbtrue, \pi) + \mathcal{E}_{model},
\end{aligned}
    $%    
}
\end{equation*}
where $\mathcal{E}_{model}$ is the error associated with the maximum likelihood training, and is independent of the policy $\pi$.
\end{theorem}
If the function class $\mathcal{P}$ for the estimated dynamic $\widehat P$ is rich enough and we have sufficiently many empirical samples from $\dbtrue$, under the classical statistical regularity condition, we can achieve a small model error $\mathcal{E}_{model}$ by MLE \citep{ferguson1996,casellaberger2001}. 
So we focus on minimizing the first two terms in Theorem~\ref{col:sum_up} as a regularization during the policy improvement step.

\subsection{Practical Implementation} \label{sec:main_method_details}
%\yf{@shentao, can you add some textbf to summary each part for this subsection, the whole subsection are too long for this part and people may loose information, e.g: the first part can add something like model training or model design?}
Based on the analysis in \Secref{sec:method_theory}, we would like to minimize $D_{\cG}(\dbtrue, \dpitrue)$ as a regularization term in the policy optimization step.

\textbf{Model Training.} The first step is to minimize the model error $\mathcal{E}_{model}$ by pre-training the dynamic model under a sufficiently rich function class via MLE.
In this paper, we take on a fully-offline perspective, assuming no prior knowledge about the reward function and the termination condition.
We use neural networks to parameterize the learned transition, reward, and termination.
Following prior model-based offline-RL work \citep{mopo2020,combo2021,mabe2021}, we assume stochastic transition and reward, and use an ensemble of Gaussian probabilistic networks, denoted as $\widehat P(\cdot \given s, a)$ and $\widehat r(s, a)$, to model these distributions.
We use the same model structure, training strategy, and hyperparameters as in \citet{mopo2020}, where the model outputs the mean and log-standard-deviation of the normal distributions of the reward and next state. 
The termination condition is modelled as an ensemble of deterministic sigmoid networks with the same training strategy as $\widehat P$ and $\widehat r$, and outputs the probability of the input satisfying the termination condition.
For computational simplicity, the termination-condition networks share the input and hidden layers with $\widehat P$ and $\widehat r$.
A default cutoff-value of $0.5$ is used to decide termination.
Further details on model training is in Appendix \ref{sec:algo_details_model}.

\paragraph{Model Rollouts.} In practice, we have only a finite static dataset, and we therefore replace sampling from $\dbtrue$ with sampling from its discrete approximation $\denv$.
With the learned model $\widehat P$, we follow prior model-based RL work \citep{mbpo2019, mopo2020} to augment $\denv$ with the replay buffer $\dmodel$ consisting of $h$-horizon rollouts of the current policy $\pi_\vphi$ on the learned model $\widehat P$, by branching from states in the offline dataset $\denv$.
As discussed in \Secref{sec:background}, the augmented dataset is defined as $\gD := f \cdot \denv + (1 - f)\cdot \dmodel$, with $f \in [0,1]$ denotes the percentage of real data.
While theoretically not required in the derivation in \Secref{sec:method_theory}, in our preliminary study we find that adding a significant amount of synthetic rollouts can benefit the performance.
We hypothesize that adding $\dmodel$ can better estimate the stationary state-action distribution of the current policy and hence mitigate the issue of distribution mismatch in offline policy evaluation and optimization.
Besides, $\dmodel$ can also serve as a natural data-augmentation to $\denv$, reducing the sampling error in training the action-value function, as discussed in \citet{combo2021}.
We therefore use $\gD$ in lieu of $\denv$ in the policy evaluation, policy optimization, and regularizer construction.
We follow \citet{combo2021} to use $f=0.5$ for our algorithm and conduct an ablation study on $f$ in \Secref{sec:exp_ablation}.

\textbf{Regularizer Construction.} Though technically possible to use different test functions under different function classes to separately estimate $\mathcal{R}_{\cF}(\dbtrue, \pi,\widehat{P})$ and $\mathcal{R}_{\Psi}(\dbtrue, \pi)$,
for computational simplicity, we use the same test function to directly estimate the sum $\mathcal{R}_{\cF}(\dbtrue, \pi,\widehat{P}) + \mathcal{R}_{\Psi}(\dbtrue, \pi)$, and parametrize the test function using neural network.
%which we follow \citet{repbsde2021} to assume its richness in keep the theoretical bounds (approximately) valid.
We leave the strategy of separate estimation as future work.

A sample-based estimate of $\mathcal{R}_{\cF}(\dbtrue, \pi,\widehat{P}) + \mathcal{R}_{\Psi}(\dbtrue, \pi)$ requires sampling from the data distribution, which we dub as ``true" samples; and sampling from the policy's distribution, dubbed as ``fake" samples.
The ``true" samples $\gB_{\mathrm{true}}$ are formed by sampling from $\denv$.
To form the ``fake" samples $\gB_{\mathrm{fake}}$, we first sample $s \sim \gD$, followed by getting a corresponding action for each $s$ using $a \sim \pi_\vphi(\cdot \given s)$.
Then for each $(s,a)$ pair, we get a sample of the next state $s'$ using the learned model $\widehat P$ as $s' \sim \widehat P(\cdot \given s, a)$, followed by sampling one next action $a'$ for each $s'$ via $a' \sim \pi_\vphi(\cdot \given s')$.
Finally, we construct the ``fake" sample $\gB_{\mathrm{fake}}$ as 
\begin{equation} \label{eq:fake_sample}
\gB_{\mathrm{fake}} :=
\begin{bmatrix}
    (s~, & a~) \\
    (s', & a')
    \end{bmatrix}.
\end{equation}
Here, $(s,a)$ is used to estimate $\E_{s\sim \dbtrue\atop{a\sim \pi(\cdot|s)}}[\psi(s,a)]$, the second term inside the $\abs{\cdot}$ of $\mathcal{R}_{\Psi}(\dbtrue, \pi)$.
The pair $(s',a')$ corresponds to $\E_{{s^\prime \sim \widehat P(\cdot | s,a), a^\prime \sim \pi(\cdot | s^\prime)}\atop {(s,a)\sim \dbtrue }}\left[f(s^\prime, a^\prime)\right]$, the second term inside the absolute-value of $\mathcal{R}_{\cF}(\dbtrue, \pi,\widehat{P})$.

%\yf{@shentao, I feel we should change GAN term to JSD / dual presentation of JSD and perform minimax optimiaztion directly, GAN is not a good word in terms of presentation}
Notice that both the IPM and JSD are well-defined statistical metrics, and matching {\it w.r.t.} one of them effectively matches the other.
Meanwhile, the GAN \citep{gan2014} framework provides a stable training scheme for approximately minimizing JSD, with easy reference to hyperparameter configuration in literature.
We henceforth change the IPM structure in $\mathcal{R}_{\cF}(\dbtrue, \pi,\widehat{P}) + \mathcal{R}_{\Psi}(\dbtrue, \pi)$ to JSD and approximately minimize it via GAN.
This amounts to training a discriminator $D_\vw$ to maximize
\begin{equation}\label{eq:dis_objective}
\textstyle
    \begin{split}
        \frac{1}{\abs{\gB_{\mathrm{true}}}} & \sum_{(s,a)\sim \gB_{\mathrm{true}}}\sbr{\log D_\vw(s,a)} + \\
        & \quad \frac{1}{\abs{\gB_{\mathrm{fake}}}} \sum_{(s,a)\sim \gB_{\mathrm{fake}}} \sbr{\log\br{1-D_\vw(s,a)}},
    \end{split}
\end{equation}
and adding the generator loss $\gL_g(\vphi)$ defined as
\begin{equation} \label{eq:generator_loss}
\textstyle
    \gL_g(\vphi) := \frac{1}{\abs{\gB_{\mathrm{fake}}}} \sum_{(s,a) \in \gB_{\mathrm{fake}}}\sbr{\log\br{1-D_\vw\br{s,a}}}
\end{equation}
as the regularizer into the policy optimization objective.
In the ablation study (\Secref{sec:exp_ablation}), we compare our JSD implementation with minimizing the dual form of Wasserstein-$1$ distance, a special instance of IPM.

\textbf{Critic Training.} Motivated by \citet{bcq2019} and \citet{bear2019}, we use the following critic target:
\begin{equation} \label{eq:q_tilde}
\resizebox{.47\textwidth}{!}{%
$
     \begin{aligned}
         &\widetilde{Q}\br{s, a} \triangleq  r(s,a).\mathrm{clamp}(r_{\mathrm{min}}, r_{\mathrm{max}}) + \\
         & \gamma \E_{a'\sim \pi_{\vphi'}\br{\cdot \given s'}}\sbr{c \min_{j=1,2}Q_{\vtheta'_j}\br{s', a'} + (1-c) \max_{j=1,2}Q_{\vtheta'_j}\br{s', a'}} %
     \end{aligned}
$
}
\end{equation}
with $c=0.75$ as in prior work, $r_{\mathrm{min}} = r_0 - 3 \sigma_r$ and $r_{\mathrm{max}} = r_1 + 3\sigma_r$, where $r_0, r_1, \sigma_r$ respectively denote the minimum, maximum, and standard deviation of the rewards in the offline dataset.
Here we adopt similar %reward-clapping
reward-clipping strategy as \citet{wmopo2021} to mitigate the inaccuracy in the reward model.
Each of the double critic networks is trained to minimize the critic loss over $\br{s, a} \in \gB \sim \gD$, {\it i.e.,} $\forall\, j=1,2$,
\begin{equation}\label{eq:critic_target_main}\textstyle
\resizebox{.425\textwidth}{!}{%
$
    \min_{\vtheta_j}\frac{1}{\abs{\gB}}\sum_{\br{s, a} \in \gB} \mathrm{Huber}\br{ Q_{\vtheta_j}\br{s, a} - \widetilde{Q}(s, a)},
    $
}
\end{equation}
where we replace the usual mean-square-error with the Huber loss $\mathrm{Huber}(\cdot)$ for training stability.

\textbf{Policy Training.} We follow \citet{bear2019} to define the policy optimization objective as
\begin{equation} \label{eq:policy_target}
    \arg\min_{\vphi} -\lambda \cdot\frac{1}{\abs{\gB}} \sum_{s \in \gB, a \sim \pi_\vphi(\cdot \given s)}\sbr{\min_{j=1,2} Q_{\vtheta_j}(s, a)} + \gL_g(\vphi),
\end{equation}
where the regularization coefficient $\lambda$ is constructed similar to \citet{td3bc2021} as $$\lambda = \alpha/Q_{avg},$$ with soft-updated $Q_{avg}$.
We use $\alpha=10$ across all datasets in our main results in \Secref{sec:exp_main}.

% implicit
\textbf{Implicit Policy.} To better approximate the maximizer of the action-value function, instead of the classical deterministic or Gaussian policy, in this paper we use a richer function class, the class of implicit distributions, to define our policy.
Specifically, similar to the generator in the conditional GAN (CGAN, \citet{cgan2014}), our implicit policy uses a deterministic neural network to transform a given noise distribution $p_z(z)$ into the state-conditional action distribution. 
Concretely, given state $s$, 
\begin{equation*} \textstyle 
    a \sim \pi_\vphi(\cdot \given s) = \pi_\vphi(s, z), \text{ with } z \overset{iid}{\sim} p_z(z),
\end{equation*}
where $\pi_\vphi$ is a deterministic network.
\Secref{sec:exp_toy} conducts a toy experiment to show the efficacy of the implicit policy, particularly when fitting distributions with multimodality.

We summarize the main steps of our method Algorithm~\ref{alg:simple}, which is implemented by approximately minimizing JSD via GAN and is dubbed as Stationary Distribution Matching via GAN (SDM-GAN).
Further implementation details are provided in Appendix \ref{sec:algo_details}.

\begin{algorithm}[t]
\caption{SDM-GAN, Main Steps}
\begin{algorithmic}
\label{alg:simple}
\STATE Initialize dynamic model $\widehat P$, policy network $\pi_{\vphi}$, critic network $Q_{\vtheta_1}$ and $Q_{\vtheta_2}$, discriminator network $D_\vw$.
\STATE Train $\widehat P$ using maximum likelihood estimation Eq. \eqref{eq:model_objective}
\FOR{each iteration}
\STATE Rollout $\pi_\vphi$ on $\widehat P$, add the synthetic data into $\dmodel$.
\STATE Sample mini-batch $\gB \sim \gD = f\denv + (1-f) \dmodel$. 
\STATE Get critic target via Eq.~\eqref{eq:q_tilde} and train critics by Eq.~\eqref{eq:critic_target_main}.
\STATE Construct $\gB_{\mathrm{fake}}$ via Eq.~\eqref{eq:fake_sample} and sample $\gB_{\mathrm{true}}\sim \denv$.
\STATE Optimize the discriminator $D_\vw$ using Eq.~\eqref{eq:dis_objective}.
\STATE Calculate generator loss $\gL_g(\vphi)$ using Eq.~\eqref{eq:generator_loss}.
\STATE Optimize policy network $\pi_{\vphi}$ by Eq.~\eqref{eq:policy_target}.
\ENDFOR
\end{algorithmic}
\end{algorithm}

%% file: Tex/related.tex
\section{Related Work} \label{sec:related_work}
\textbf{Model-free Offline RL.}
This paper is related to a classical yet persistent idea in model-free offline RL, {\it i.e.}, matching the learned policy with the behavior policy so that the learned policy mostly visits state-action pairs similar to the offline dataset.
In literature, several designs have been proposed for this purpose, such as constraining the learned policy towards an estimated behavior policy \citep{bear2019,brac2019,offlinerldialog2019}, a tactful decomposition of the action or the action-value function into a behavior-cloning part and an offset \citep{bcq2019, fisherbrc2021}, modification of the critic-learning objective \citep{algaedice2019,cql2020}.
Our paper adds to the literature by directly matching the undiscounted stationary distribution of the learned policy towards the offline dataset. 
To achieve this, we develop the undiscounted version of the change of variable trick to derive a tractable regularizer for policy optimization, which has been explored only in the discounted case \citep{dualdice2019} in the off-policy evaluation (OPE) literature \citep{jiang16doubly,breakingcurse2018,gendice2020}.
Different from the goal of estimating the performance of some policies using offline data, our goal here is to construct a tractable regularization for policy optimization, which involves a learned transition model trained by MLE.
{\citet{jointmatching2022} recently propose the design of an implicit policy via the CGAN structure, together with several additional enhancing techniques. Our paper improves on this prior work by adopting a model-based approach, which facilitates regularizer construction and mitigates distribution mismatch.}
%OPE literature, and use a learned transition dynamic trained by MLE to 
%implement our method.
%To achieve this, we adopt the DualDICE technique in the OPE literature \citep{dualdice2019} to derive a tractable regularizer for policy optimization, 

\textbf{Model-based Offline RL.}
Similar notion of guiding the policy away from the OOD state-action pairs also persists in model-based offline RL.
In literature, such guidance has been implemented by model-output-based uncertainty quantification \citep{mopo2020, morel2020}, estimation of a conservative action-value function \citep{combo2021}, and regularizing the model-based policy learning with a possibly-adaptive behavior prior \citep{mabe2021,bremen2021}.
Recently, techniques in the OPE literarture has been adopted into policy training and model learning.
\citet{wmopo2021} propose an EM-style iterative training of model and policy, with a density-ratio-based weighted model estimation.
Similar to our work, \citet{repbsde2021} utilize the change of variable technique to train the dynamic, which can potentially improve 
the quality of the representations of the dynamic during the model training. 
%DualDICE trick, but focuses on balancing between the discounted stationary distribution induced by the learned policy and the data distribution, for learning distribution-shift-invariant representation on the state-action space during the model training.
By contrast, our paper uses the model-generated rollouts to implement a tractable bound for directly constraining the undiscounted stationary distribution of the learned policy towards the offline dataset during the policy optimization step.
For algorithmic simplicity, our paper does not apply the technique of uncertainty quantification, learnable behavior prior, advanced design of the dynamic model via the GPT-Transformer \citep{gpt2018,fan2020bayesian,trajectorytransformer2021, zhang2022allsh}, {\it etc.}
These are orthogonal to our method and may further improve the empirical performance.
%\textbf{Off-policy Evaluation (OPE):}

%\textbf{Distribution Matching:}

%% file: Tex/exp.tex
\section{Experiments} \label{sec:experiment}
\begin{figure}[t]
     % \vspace{+2mm}
     \centering
     \begin{subfigure}[b]{0.11\textwidth}
         \centering
         \includegraphics[width=\textwidth]{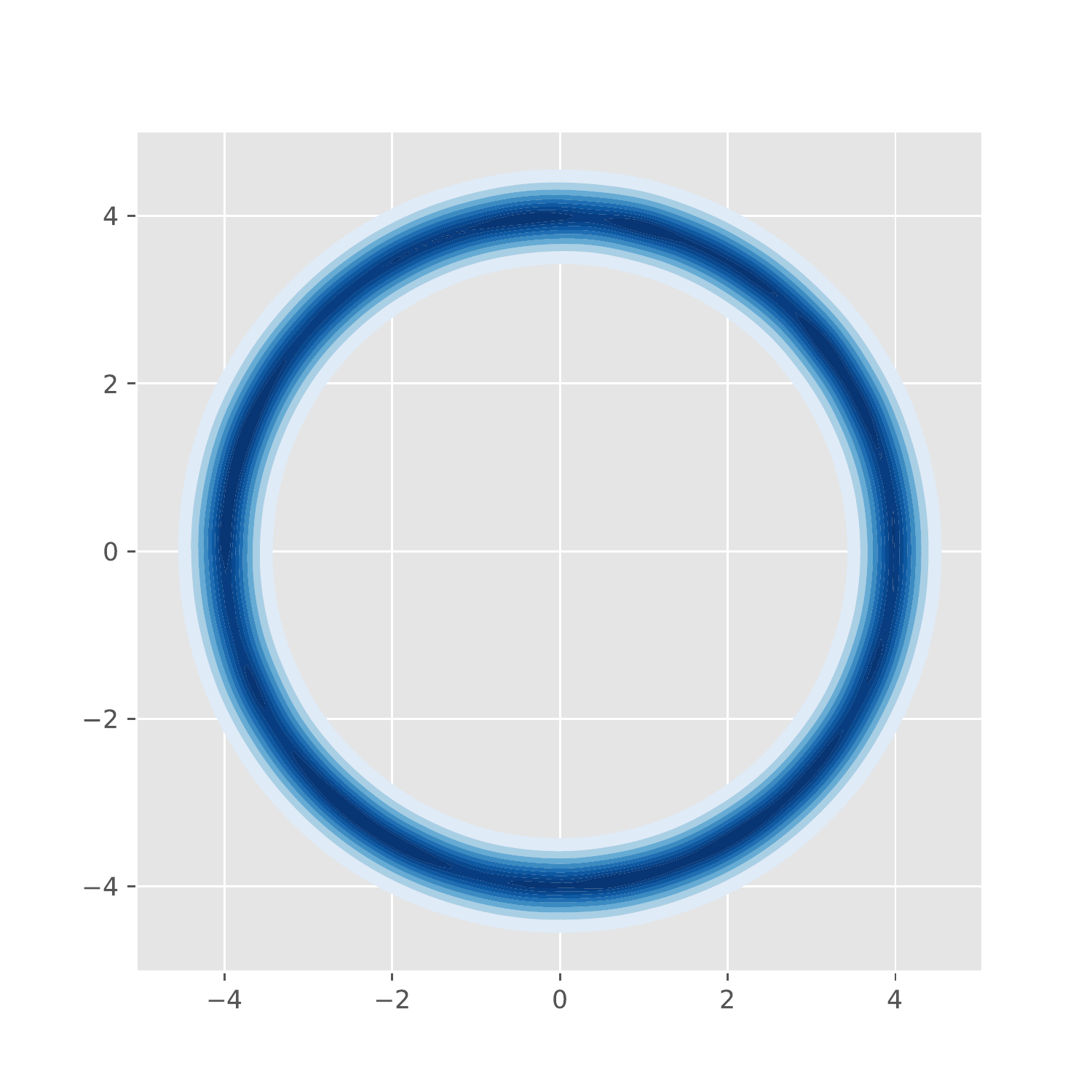}
         \caption{Truth}
         \label{fig:toy_truth}
     \end{subfigure}
     \hfill
     \begin{subfigure}[b]{0.11\textwidth}
         \centering
         \includegraphics[width=\textwidth]{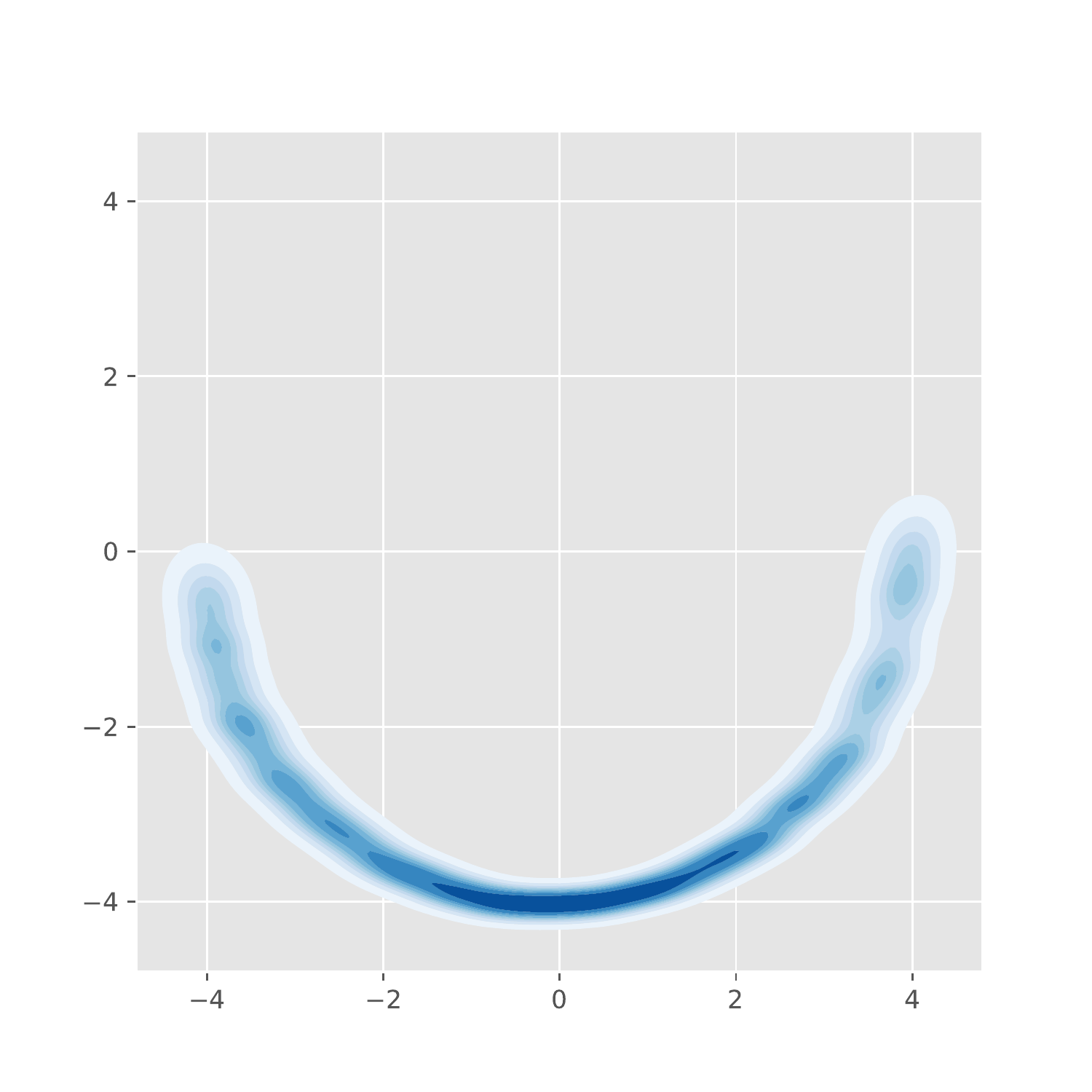}
         \caption{D-CGAN}
         \label{fig:toy_dcgan}
     \end{subfigure}
     \hfill
     \begin{subfigure}[b]{0.11\textwidth}
         \centering
         \includegraphics[width=\textwidth]{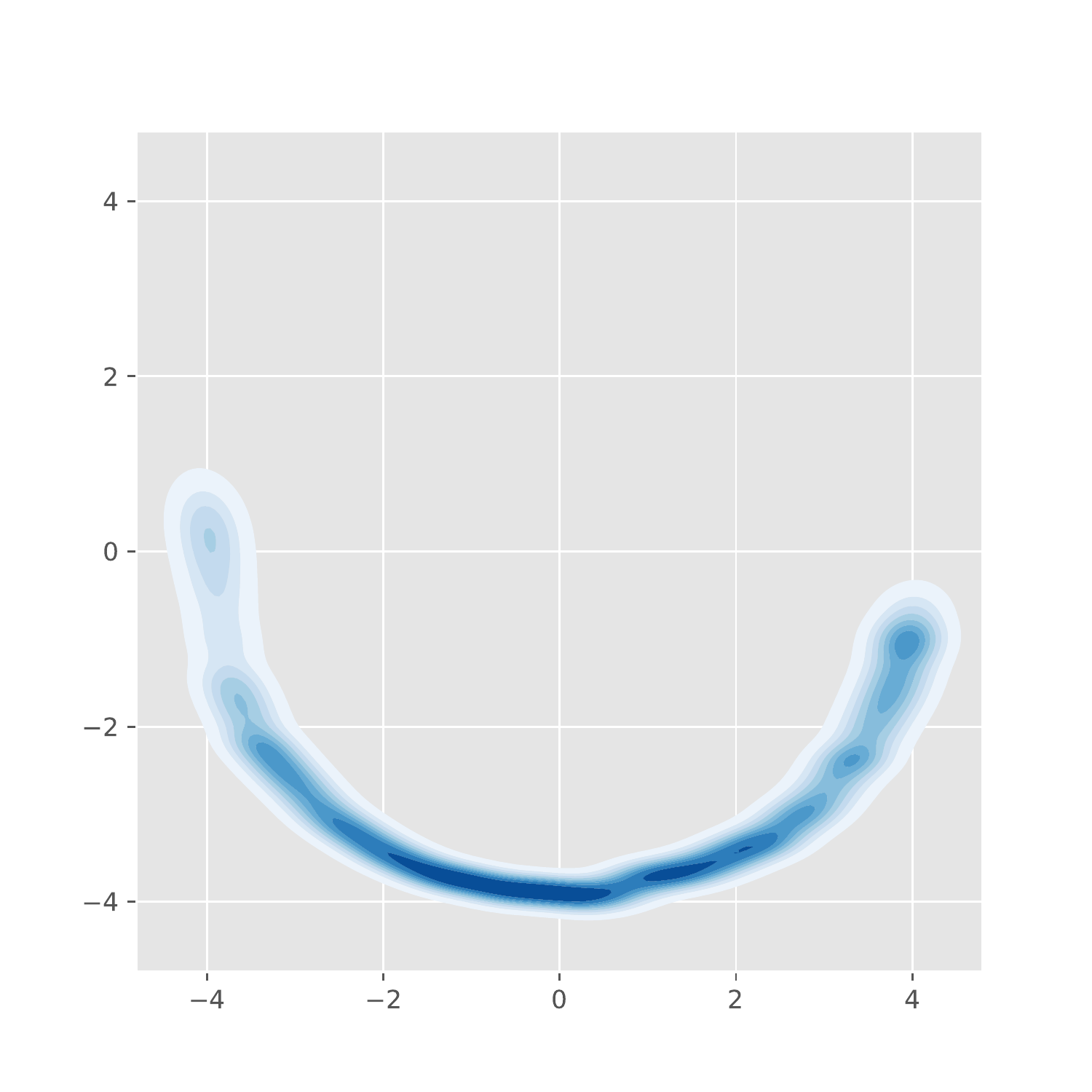}
         \caption{G-CGAN}
         \label{fig:toy_gcgan}
     \end{subfigure}
     \hfill
     \begin{subfigure}[b]{0.11\textwidth}
         \centering
         \includegraphics[width=\textwidth]{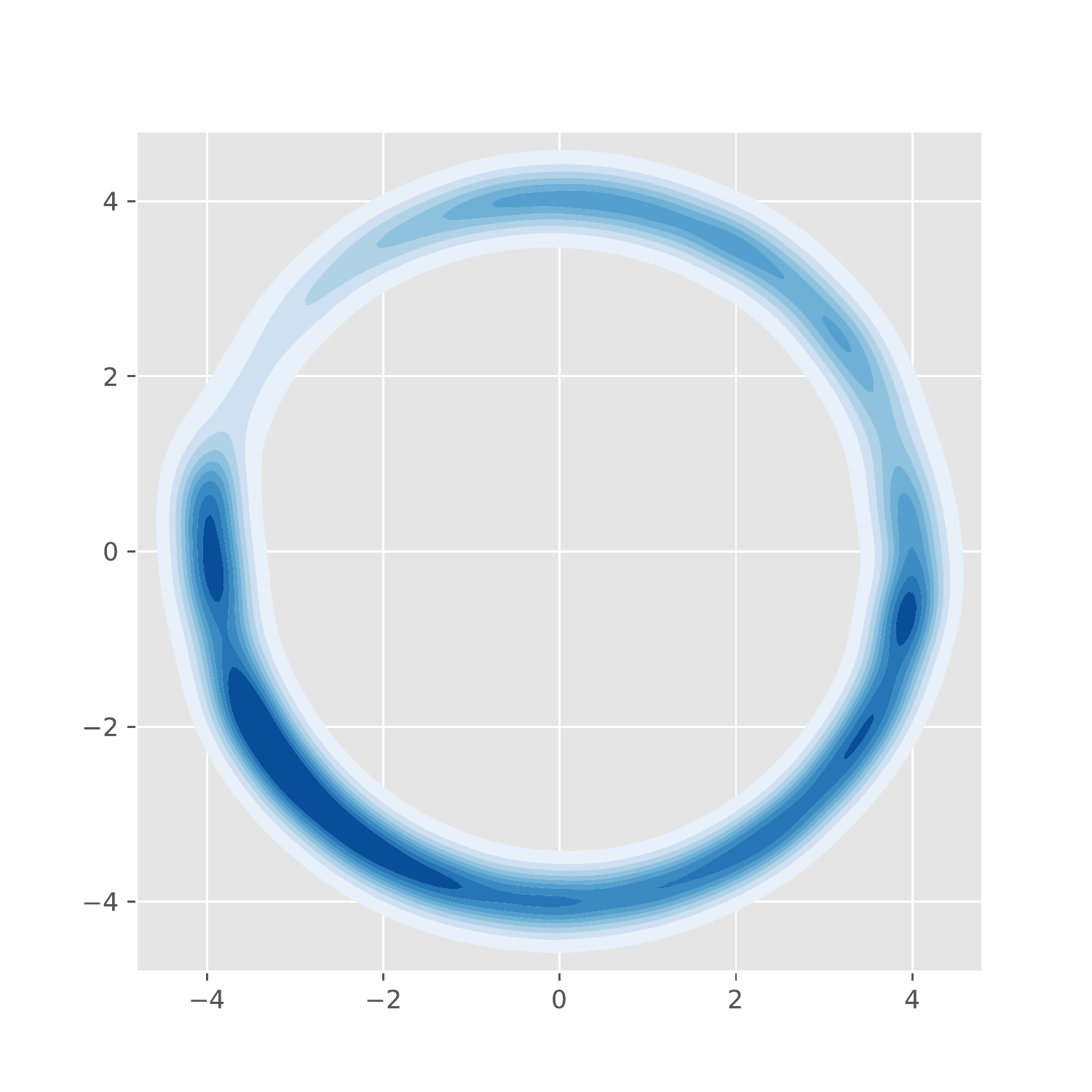}
         \caption{CGAN}
         \label{fig:toy_cgan}
     \end{subfigure}
     \vspace{-3mm}
        \caption{ 
        Performance of approximating the behavior policy on the circle dataset. 
        A deterministic-generator conditional GAN (``D-CGAN"), a Gaussian-generator conditional GAN (“G-CGAN”) and a conditional GAN (``CGAN") are fitted using the classical policy (conditional distribution) matching approach similar to \citet{brac2019}.
        More details are provided in Appendix \ref{sec:toy_details}.
        }
        \label{fig:toy}
        \vspace{-6.5mm}
\end{figure}

In what follows, we first show the effectiveness of an implicit policy through a toy example (\Secref{sec:exp_toy}).
We then present an empirical evaluation of our algorithm (\Secref{sec:exp_main}), followed by an ablation study (\Secref{sec:exp_ablation}).
{Source code for the experiments is publicly \href{https://github.com/Shentao-YANG/SDM-GAN_ICML2022}{available}.}

%\yf{make the figure smaller, it is too alarge and has too much space}

\subsection{Toy Experiments}\label{sec:exp_toy}

As a motivation to train a flexible implicit policy instead of the classical deterministic or Gaussian policy, we conduct a toy behavior-cloning experiment, as shown on Figure~\ref{fig:toy}.
We remark that the simplest form of offline behavior cloning is essentially a supervised learning task, with the reward function and the termination condition being unnecessary.
Throughout Figure~\ref{fig:toy}, we use the $x$-axis to represent the state and $y$-axis the action.
Figure~\ref{fig:toy_truth} plots the state-action distribution that we seek to clone, which presents multimodality on almost all states.
Such an offline dataset may come from multiple data-collecting policies or from the multiple maxima of the action-value function on the action space.
We hence expect the learned policy to capture the targeted state-action distribution, which can translate into a better capture of the rewarding actions and an improved generalization beyond the static dataset in offline RL tasks.
Figures~\ref{fig:toy_dcgan}-\ref{fig:toy_cgan} plot the state-action distributions of the learned policies on the testset, in which we compare the implicit policy implemented by conditional GAN (Figure~\ref{fig:toy_cgan}) with its variants of changing the implicit generator to deterministic transformation (Figure~\ref{fig:toy_dcgan}) and to the Gaussian generator (Figure~\ref{fig:toy_gcgan}).
Experimental details are discussed in Appendix~\ref{sec:toy_details}.

We see from Figures~\ref{fig:toy_dcgan} and \ref{fig:toy_gcgan} that both the deterministic and the Gaussian policy fail to capture necessary action modes, which directly relates to their less flexible structures.
By contrast, as shown in Figure~\ref{fig:toy_cgan}, the implicit policy defined by conditional GAN does capture all the action modes on each state and fits well onto the targeted distribution.
This toy example motivates our algorithmic design of training an implicit policy via the conditional GAN structure.

\subsection{Main Results}\label{sec:exp_main}
\begin{table*}[h] 
 \vspace{-1mm}
\caption{
Normalized returns for the experiments on the D4RL datasets.
We use the ``v0” version of the datasets in the Gym-MuJoCo and Adroit domains.
On the Gym-Mojoco domain, we bold both the highest score over all algorithms and the highest score of model-based  algorithms. 
On the Maze2D and Adroit tasks, we bold the best scores. 
} 
\label{table:main} 
\centering 
\def\arraystretch{1.}
\resizebox{\textwidth}{!}{
\begin{tabular}{lccccccccc}
\toprule
                 Task Name &  aDICE &                               CQL &                        FisherBRC &                           TD3+BC &                          OptiDICE &                                     MOPO &                                    COMBO &                                WMOPO &                           SDM-GAN \\
\midrule
        halfcheetah-medium &   -2.2 &    39.0 $\pm$ {\footnotesize 0.8} &   41.1 $\pm$ {\footnotesize 0.6} &   43.0 $\pm$ {\footnotesize 0.5} &    38.2 $\pm$ {\footnotesize 0.5} &           47.2 $\pm$ {\footnotesize 1.0} &           53.7 $\pm$ {\footnotesize 2.1} &           \textbf{55.6} $\pm$ {\footnotesize 1.3} &    42.5 $\pm$ {\footnotesize 0.5} \\
           walker2d-medium &    0.3 &   60.2 $\pm$ {\footnotesize 30.8} &   \textbf{78.4} $\pm$ {\footnotesize 1.8} &   77.3 $\pm$ {\footnotesize 4.0} &   14.3 $\pm$ {\footnotesize 15.0} &            0.0 $\pm$ {\footnotesize 0.1} &          40.9 $\pm$ {\footnotesize 28.9} &          22.7 $\pm$ {\footnotesize 27.7} &    \textbf{66.7} $\pm$ {\footnotesize 1.8} \\
             hopper-medium &    1.2 &   34.5 $\pm$ {\footnotesize 11.7} &   99.2 $\pm$ {\footnotesize 0.3} &   \textbf{99.6} $\pm$ {\footnotesize 0.6} &   92.3 $\pm$ {\footnotesize 16.9} &           23.4 $\pm$ {\footnotesize 7.2} &          51.8 $\pm$ {\footnotesize 32.8} &          \textbf{66.5} $\pm$ {\footnotesize 46.0} &   62.8 $\pm$ {\footnotesize 14.3} \\
 halfcheetah-medium-replay &   -2.1 &    43.4 $\pm$ {\footnotesize 0.8} &   43.2 $\pm$ {\footnotesize 1.3} &   41.9 $\pm$ {\footnotesize 2.0} &    39.8 $\pm$ {\footnotesize 0.8} &           \textbf{52.5} $\pm$ {\footnotesize 1.4} &           51.8 $\pm$ {\footnotesize 1.6} &           51.8 $\pm$ {\footnotesize 5.6} &    41.7 $\pm$ {\footnotesize 0.4} \\
    walker2d-medium-replay &    0.6 &    16.4 $\pm$ {\footnotesize 6.6} &  38.4 $\pm$ {\footnotesize 16.6} &   24.6 $\pm$ {\footnotesize 6.7} &    20.2 $\pm$ {\footnotesize 5.8} &          51.9 $\pm$ {\footnotesize 15.8} &          14.2 $\pm$ {\footnotesize 11.9} &          \textbf{54.8} $\pm$ {\footnotesize 12.3} &    20.3 $\pm$ {\footnotesize 4.0} \\
      hopper-medium-replay &    1.1 &    29.5 $\pm$ {\footnotesize 2.3} &   33.4 $\pm$ {\footnotesize 2.8} &   31.4 $\pm$ {\footnotesize 2.7} &    29.0 $\pm$ {\footnotesize 4.9} &          47.1 $\pm$ {\footnotesize 16.2} &           34.5 $\pm$ {\footnotesize 2.0} &           \textbf{93.9} $\pm$ {\footnotesize 1.9} &    30.6 $\pm$ {\footnotesize 2.8} \\
 halfcheetah-medium-expert &   -0.8 &   34.5 $\pm$ {\footnotesize 15.8} &   \textbf{92.5} $\pm$ {\footnotesize 8.5} &   90.1 $\pm$ {\footnotesize 6.9} &   91.2 $\pm$ {\footnotesize 16.6} &           \textbf{92.1} $\pm$ {\footnotesize 8.3} &          90.0 $\pm$ {\footnotesize 10.5} &          42.7 $\pm$ {\footnotesize 13.0} &    89.1 $\pm$ {\footnotesize 6.6} \\
    walker2d-medium-expert &    0.4 &   79.8 $\pm$ {\footnotesize 22.7} &  \textbf{98.2} $\pm$ {\footnotesize 13.1} &  96.1 $\pm$ {\footnotesize 15.8} &   67.1 $\pm$ {\footnotesize 30.2} &          36.0 $\pm$ {\footnotesize 49.6} &          61.3 $\pm$ {\footnotesize 36.1} &          48.6 $\pm$ {\footnotesize 37.0} &    \textbf{97.9} $\pm$ {\footnotesize 4.9} \\
      hopper-medium-expert &    1.1 &  103.5 $\pm$ {\footnotesize 20.2} &  112.3 $\pm$ {\footnotesize 0.3} &  111.9 $\pm$ {\footnotesize 0.3} &  101.8 $\pm$ {\footnotesize 18.5} &           27.8 $\pm$ {\footnotesize 3.6} &          \textbf{112.6} $\pm$ {\footnotesize 1.8} &          97.8 $\pm$ {\footnotesize 19.3} &   104.5 $\pm$ {\footnotesize 5.4} \\
              maze2d-large &   -0.1 &   43.7 $\pm$ {\footnotesize 18.6} &   -2.1 $\pm$ {\footnotesize 0.4} &  87.6 $\pm$ {\footnotesize 15.4} &  130.7 $\pm$ {\footnotesize 56.1} &  - &  - &  - &  \textbf{207.7} $\pm$ {\footnotesize 11.7} \\
             maze2d-medium &   10.0 &    30.7 $\pm$ {\footnotesize 9.8} &   4.6 $\pm$ {\footnotesize 20.4} &  59.1 $\pm$ {\footnotesize 47.7} &  \textbf{140.8} $\pm$ {\footnotesize 44.0} &  - &  - &  - &  115.4 $\pm$ {\footnotesize 34.2} \\
              maze2d-umaze &  -15.7 &    50.5 $\pm$ {\footnotesize 7.9} &  -2.3 $\pm$ {\footnotesize 17.9} &  13.8 $\pm$ {\footnotesize 22.8} &  \textbf{107.6} $\pm$ {\footnotesize 33.1} &  - &  - &  - &   36.1 $\pm$ {\footnotesize 28.4} \\
                 pen-human &   -3.3 &    2.1 $\pm$ {\footnotesize 13.7} &    0.0 $\pm$ {\footnotesize 3.9} &   -1.7 $\pm$ {\footnotesize 3.8} &    -0.1 $\pm$ {\footnotesize 5.6} &  - &  - &  - &    \textbf{17.8} $\pm$ {\footnotesize 1.7} \\
                pen-cloned &   -2.9 &     1.5 $\pm$ {\footnotesize 6.2} &   -2.0 $\pm$ {\footnotesize 0.8} &   -2.4 $\pm$ {\footnotesize 1.4} &     1.4 $\pm$ {\footnotesize 6.8} &  - &  - &  - &    \textbf{40.6} $\pm$ {\footnotesize 6.1} \\
                pen-expert &   -3.5 &   95.9 $\pm$ {\footnotesize 18.1} &  31.6 $\pm$ {\footnotesize 24.4} &  32.4 $\pm$ {\footnotesize 24.3} &    -1.1 $\pm$ {\footnotesize 4.7} &  - &  - &  - &  \textbf{135.8} $\pm$ {\footnotesize 11.7} \\
               door-expert &    0.0 &   87.9 $\pm$ {\footnotesize 21.6} &  57.6 $\pm$ {\footnotesize 37.7} &   -0.3 $\pm$ {\footnotesize 0.0} &   87.9 $\pm$ {\footnotesize 25.8} &  - &  - &  - &    \textbf{93.5} $\pm$ {\footnotesize 6.7} \\
\bottomrule
\end{tabular}
}
\vspace{-4mm}
\end{table*}

%\yf{Can we add some summary paragraph textbf? it is a little bit too long for readers}
As discussed in \Secref{sec:main_method_details}, we implement our algorithm by approximately matching the JSD between the stationary distributions via the conditional GAN, which we dub as Stationary Distribution Matching via GAN (SDM-GAN).
We compare our algorithm with three state-of-the-art (SOTA) model-free offline-RL algorithms: CQL \citep{cql2020}, FisherBRC \citep{fisherbrc2021}, and TD3+BC \citep{td3bc2021}; three SOTA model-based offline-RL algorithms: MOPO \citep{mopo2020}, COMBO \citep{combo2021}, and WMOPO \citep{wmopo2021};
{and two ``DICE-style" algorithms: AlgaeDICE \citep{algaedice2019} and OptiDICE \citep{optidice2021}}.
For WMOPO, we use the version of $\alpha = 0.2$, which is the paper's proposal.
Experimental details and hyperparameter setups are discussed in detail in Appendix~\ref{sec:rl_exp_details}.
We remark that unlike our algorithm, the three model-based RL baselines assume known termination function and possibly also known reward function, which potentially limits their applicability.
% rerun the result
{
Except AlgaeDICE, we rerun each baseline method using the official source code under the recommended hyperparameters.
The results of AlgaeDICE (``aDICE") is obtained from the D4RL whitepaper \citep{fu2021d4rl}.
}
% three domain of tasks, mbrl do not have maze2d and adroit
We compare our algorithm with the baselines on a diverse set of continuous-control offline-RL datasets (discussed in detail in Appendix~\ref{sec:rl_exp_details}), ranging across the Gym-Mojoco, Maze2D, and Adroit domains in the D4RL benchmark.
Note that the three model-based RL baselines do not provide hyperparameter configurations for the Maze2D and Adroit domains, and therefore we do not evaluate them on these datasets.
Table~\ref{table:main} shows the mean and standard deviation of each algorithm on each tested dataset.

In Table~\ref{table:main}, our algorithm shows competitive and relatively-stable performance across a diverse array of datasets.
Specifically, compared to the baselines, our algorithm shows better and more robust performance on the two task domains that are traditionally considered as challenging in offline-RL  \citep{fu2021d4rl}: 
the Adroit domain featured by high dimensionality and narrow data-distribution; and the Maze2D domain collected by non-Markovian policies.
On these two domains, behavioral-cloning (BC) based algorithms, such as FisherBRC and TD3+BC, are likely to fail. 
Concretely, on the Adroit datasets, a main difficulty is estimating the behavior policy and regularizing the current policy towards it in a high-dimensional space with limited data. 
On the Maze2D datasets, policy learning is challenged by the error of using a Markovian policy to clone the non-Markovian data-collecting policy.
The complexity of these datasets and underlying environments may also call for a more flexible policy class, and thus excels our design of an implicit policy over the classical deterministic or Gaussian policy.
On the Gym-Mojoco domain, our algorithm is able to learn from median-quality datasets and from samples collected by a mixture of data-collecting policies, which are closer to the real-world settings.

{
Comparing SDM-GAN with AlgaeDICE and OptiDICE reveals the overall benefit of our approach over the ``DICE-style" method.
As an example, OptiDICE requires BC, which is implemented by a Gaussian-mixture policy with several components.
While this BC strategy may be beneficial on relatively-lower dimensional Maze2D datasets, it may not generalize onto the higher-dimensional Adroit datasets. 
This may be a direct result of the previously-stated difficulty of BC-based methods on the Adroit datasets.
}

Additionally, we notice that some of the designs in these SOTA baselines are orthogonal to our algorithm, such as a tactful estimation of the action-value function (CQL, FisherBRC, COMBO), an incorporation of uncertainty quantification into the estimated dynamic (MOPO), and a density-ratio-based weighted model estimation (WMOPO).
These designs may be easily incorporated into our algorithm, leaving future work and potentials for further improvement.
% We leave such an investigation for future work.

\subsection{Ablation Study} \label{sec:exp_ablation}

\begin{figure*}[ht]
     \begin{subfigure}[b]{0.33\textwidth}
         \centering
         \includegraphics[width=\textwidth]{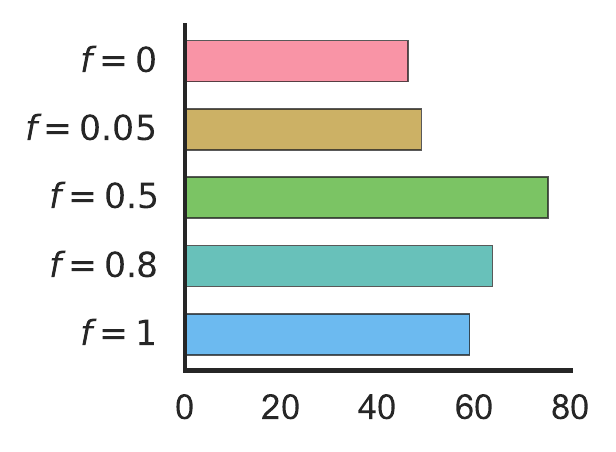}
        %  \vspace{-7mm}
         \caption{Varying $f \in \cbr{0, 0.05, 0.5, 0.8, 1}$}
         \label{fig:varying_f}
     \end{subfigure}
     \hfill
     \begin{subfigure}[b]{0.33\textwidth}
         \centering
         \includegraphics[width=\textwidth]{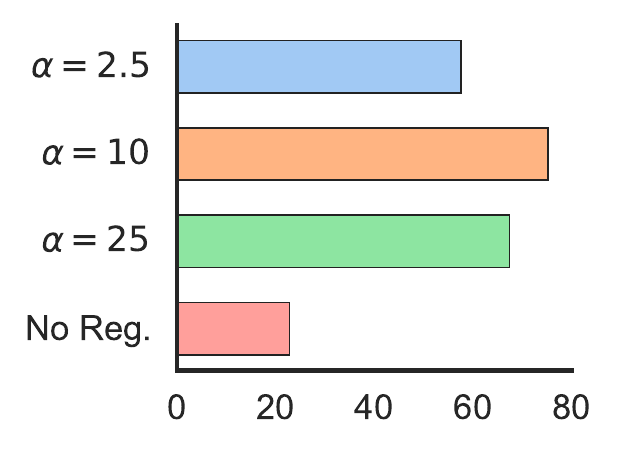}
         \caption{Varying $\alpha \in \cbr{2.5, 10, 25}$ and No Reg.}
         \label{fig:varying_alpha}
     \end{subfigure}
     \hfill
     \begin{subfigure}[b]{0.33\textwidth}
         \centering
         \includegraphics[width=\textwidth]{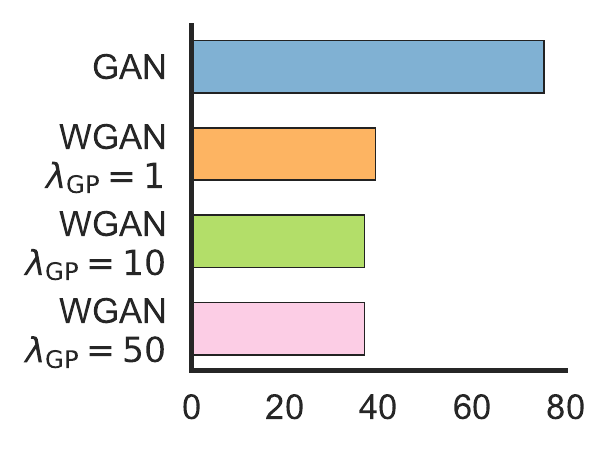}
         \caption{JSD {\it v.s.} IPM (Wasserstein-$1$ Dual)}
         \label{fig:jsd_imp}
     \end{subfigure}
     \vspace{-8mm}
        \caption{ 
        Average score ($x$-axis) over the tested datasets for each of the comparison in the ablation study (\Secref{sec:exp_ablation}).
        }
        \label{fig:ablation}
        \vspace{-6mm}
\end{figure*}

In the ablation study we seek to answer the following questions:
\textbf{(a)} Does the rollout dataset $\dmodel$ help the performance?
\textbf{(b)} Is our algorithm sensitive to the choice of the regularization coefficient $\alpha$?
\textbf{(c)} Does approximately matching the JSD between the stationary distributions perform better than matching the IPM?
Unless stated otherwise, experimental details for all algorithmic variants follow the main results (stated in Appendix~\ref{sec:rl_exp_details}).

\textbf{(a).} To investigate whether rollout data are helpful for policy training, we follow \citet{combo2021} to compare our algorithm used in \Secref{sec:exp_main} ($f=0.5$) with its variant of fewer and more rollout data, respectively $f=0.8$ and $f=0.05$, with all other settings kept the same.
{Additionally, we consider the variants of no rollout data ($f=1$) and only rollout data ($f=0$).}

Figure~\ref{fig:varying_f} plots the average score over the tested datasets for each variant, with detail comparison provided in Table~\ref{table:f_0508}.
On $11$ out of $16$ datasets, the $f=0.5$ version has higher mean returns than the $f=0.8$ version, showing the efficacy of using more rollout data.
This empirical result aligns with our intuition that rollout data can better approximate the stationary state-action distribution of the learned policy, especially when the learned policy  deviates from the behavior.
{The efficacy of model-based rollout is further corroborated by the generally worse performance of the $f=1$ variant, compared with $f=0.8$ and $f=0.5$.
Nevertheless, using too-many rollout data can be harmful, as shown by the comparison between the $f=0$, $f=0.05$, and $f=0.5$ variants.
In Table~\ref{table:f_0508}, the scores for the $f=0$ and $f=0.05$ variants are generally lower than $f=0.5$, and the standard deviations relative to the scores are generally higher.
This comparison shows the gain of adding offline data for combating the inaccuracy in the estimated dynamic.}
% This is likely due to the inaccuracy in the estimated dynamics.

\textbf{(b).} 
We test our algorithm on both a smaller and a larger value of the regularization coefficient $\alpha \in \cbr{2.5, 25}$, {together with the variant without the proposed regularization term (denoted as ``No Reg.")}.
Figure~\ref{fig:varying_alpha} plots the average scores and Table~\ref{table:alpha_sweep} presents the results.
We see that on a large portion of datasets, our method is relatively robust to the choice of $\alpha$, especially when using a larger value of $\alpha$.

A small value of $\alpha$, \textit{e.g.}, $\alpha=2.5$ in \citet{td3bc2021}, can be too conservative for our algorithm, limiting the improvement of the policy.
A larger value of $\alpha=25$, which puts more weight on policy optimization than the default $\alpha=10$, overall does not lead to significant deterioration in the performance, though the results do in general possess higher variance.
{
The inferior performance of the ``No Reg." variant shows the necessity of proper regularization, especially on the Maze2D and Adroit datasets. 
}

\textbf{(c).} 
We compare our default algorithm of approximately matching the JSD between the stationary distributions (SDM-GAN) with the variant of approximately matching the dual form of the Wasserstein-$1$ metric, which is an instance of the IPM \citep{ipm1997,ipm2009} and is dubbed as ``SDM-WGAN.''
SDM-WGAN is implemented by changing the CGAN structure in SDM-GAN with the WGAN-GP structure \citep{wgangp2017}.
We vary the strength of the Lipschitz-$1$ constraint $\lambda_{\mathrm{GP}}$ in WGAN-GP over $\lambda_{\mathrm{GP}} \in \cbr{1, 10, 50}$, with $10$ being the original default.
All other hyperparameters of SDM-WGAN remain the same as in SDM-GAN.
Detail results are presented in Table~\ref{table:jsdipm}.

From the plot of average scores, Figure~\ref{fig:jsd_imp}, we see that the performance of SDM-WGAN is relatively robust to the choice of $\lambda_{\mathrm{GP}}$.
Compared with SDM-GAN, the three variants of SDM-WGAN overall perform worse, and the performances exhibit relatively larger variation across datasets.
Furthermore, the performances of SDM-WGAN are less stable than SDM-GAN, as shown by the relatively larger variance in the results on many tested datasets.
We believe that a thorough guidance to the hyperparameter setups and training schemes for WGAN-GP could potentially improve the performance of SDM-WGAN.
Such guidance is currently missing from the literature.
The investigation onto other instances of IPM, such as the Maximum Mean Discrepancy \citep{kerneltwosampletest2012}, is left as future work.

\begin{table}[ht] 
 \vspace{-6mm}
\caption{
Normalized returns for comparing the default algorithm with variants of approximately matching the Wasserstein-$1$ dual via WGAN-GP.
We vary $\lambda_{\mathrm{GP}} \in \cbr{1, 10, 50}$.
Here ``med" denotes medium, ``rep" for replay, and ``exp" for expert.
} 
\label{table:jsdipm} 
\centering 
\def\arraystretch{1.}
\resizebox{.48\textwidth}{!}{
\begin{tabular}{lcccc}
\toprule
                 Tasks/SDM Variants &           GAN & $\lambda_{\mathrm{GP}} = 1$ & $\lambda_{\mathrm{GP}} = 10$ & $\lambda_{\mathrm{GP}} = 50$ \\
\midrule
              maze2d-umaze &   36.1 {\footnotesize $\pm$ 28.4} &                      26.7 {\footnotesize $\pm$ 42.2} &                       28.4 {\footnotesize $\pm$ 45.2} &                       56.7 {\footnotesize $\pm$ 15.0} \\
             maze2d-med &  115.4 {\footnotesize $\pm$ 34.2} &                       99.8 {\footnotesize $\pm$ 8.0} &                       75.1 {\footnotesize $\pm$ 10.9} &                       96.0 {\footnotesize $\pm$ 10.1} \\
              maze2d-large &  207.7 {\footnotesize $\pm$ 11.7} &                      26.2 {\footnotesize $\pm$ 30.0} &                       44.3 {\footnotesize $\pm$ 63.9} &                         1.5 {\footnotesize $\pm$ 9.5} \\
        halfcheetah-med &    42.5 {\footnotesize $\pm$ 0.5} &                       44.3 {\footnotesize $\pm$ 1.0} &                        45.5 {\footnotesize $\pm$ 0.2} &                        44.4 {\footnotesize $\pm$ 0.6} \\
           walker2d-med &    66.7 {\footnotesize $\pm$ 1.8} &                       9.6 {\footnotesize $\pm$ 12.7} &                       21.6 {\footnotesize $\pm$ 10.7} &                       11.7 {\footnotesize $\pm$ 12.7} \\
             hopper-med &   62.8 {\footnotesize $\pm$ 14.3} &                        9.8 {\footnotesize $\pm$ 9.2} &                         3.5 {\footnotesize $\pm$ 2.9} &                         2.7 {\footnotesize $\pm$ 3.4} \\
 halfcheetah-med-rep &    41.7 {\footnotesize $\pm$ 0.4} &                       49.3 {\footnotesize $\pm$ 0.9} &                        49.3 {\footnotesize $\pm$ 0.6} &                        48.6 {\footnotesize $\pm$ 0.5} \\
    walker2d-med-rep &    20.3 {\footnotesize $\pm$ 4.0} &                      16.2 {\footnotesize $\pm$ 10.7} &                         6.6 {\footnotesize $\pm$ 8.4} &                         8.0 {\footnotesize $\pm$ 6.2} \\
      hopper-med-rep &    30.6 {\footnotesize $\pm$ 2.8} &                       37.2 {\footnotesize $\pm$ 3.1} &                        33.3 {\footnotesize $\pm$ 4.2} &                        28.1 {\footnotesize $\pm$ 3.7} \\
 halfcheetah-med-exp &    89.1 {\footnotesize $\pm$ 6.6} &                       17.1 {\footnotesize $\pm$ 4.7} &                        15.9 {\footnotesize $\pm$ 3.9} &                        22.1 {\footnotesize $\pm$ 6.1} \\
    walker2d-med-exp &    97.9 {\footnotesize $\pm$ 4.9} &                      66.2 {\footnotesize $\pm$ 18.7} &                       56.6 {\footnotesize $\pm$ 23.9} &                        55.3 {\footnotesize $\pm$ 7.8} \\
      hopper-med-exp &   104.5 {\footnotesize $\pm$ 5.4} &                       25.5 {\footnotesize $\pm$ 2.5} &                        25.6 {\footnotesize $\pm$ 5.0} &                        21.3 {\footnotesize $\pm$ 6.3} \\
                 pen-human &    17.8 {\footnotesize $\pm$ 1.7} &                       27.4 {\footnotesize $\pm$ 2.6} &                        24.0 {\footnotesize $\pm$ 7.8} &                        18.7 {\footnotesize $\pm$ 2.9} \\
                pen-cloned &    40.6 {\footnotesize $\pm$ 6.1} &                      46.9 {\footnotesize $\pm$ 18.1} &                       44.0 {\footnotesize $\pm$ 17.4} &                       44.1 {\footnotesize $\pm$ 11.3} \\
                pen-expert &  135.8 {\footnotesize $\pm$ 11.7} &                      63.7 {\footnotesize $\pm$ 26.7} &                       67.1 {\footnotesize $\pm$ 20.7} &                       72.4 {\footnotesize $\pm$ 22.9} \\
               door-expert &    93.5 {\footnotesize $\pm$ 6.7} &                      63.5 {\footnotesize $\pm$ 27.4} &                       50.2 {\footnotesize $\pm$ 36.4} &                       59.1 {\footnotesize $\pm$ 28.8} \\
               \midrule
             Average Score &              75.2 &                                 39.3 &                                  36.9 &                                  36.9 \\
\bottomrule

\end{tabular}
}
% \vspace{-5mm}
\end{table}

%% file: Tex/conclusion.tex
\section{Conclusion} \label{sec:conclusion}
In this paper, we directly approach a central difficulty in offline RL, {\it i.e.}, the mismatch between the distribution of the offline dataset and the undiscounted stationary state-action distribution of the learned policy.
Specifically, we derive a tractable bound for the difference between the stationary distribution of the behavior policy in the underlying environmental dynamic and that of the learned policy {\it w.r.t.} the IPM.
Based on the theoretical derivation, we design a practical model-based algorithm that uses this tractable bound as a regularizer in the policy improvement step.
Our practical implementation is tested on a diverse array of continuous-control offline-RL datasets and shows competitive performance compared with several state-of-the-art model-based and model-free baselines.
Ablation study is also provided to validate some of the key algorithmic choices.

\section*{Acknowledgements}
The authors acknowledge the support of NSF IIS1812699 and ECCS-1952193, the support of a gift fund from ByteDance
Inc., and the Texas Advanced Computing Center (TACC)
for providing HPC resources that have contributed to the
research results reported within this paper.

%% file: Tex/appendix.tex
\section{Additional Tables} \label{sec:additional_tables}
Tables~\ref{table:f_0508} - \ref{table:alpha_sweep} correspond to the results for the first two ablation study in \Secref{sec:exp_ablation}.

    \begin{table}[H] 
\caption{Normalized returns for comparing the performance of our algorithm under different real data percentage $f \in \cbr{0,0.05,0.5, 0.8,1}$.
The reported numbers are the means and standard deviation across three random seeds $\cbr{0,1,2}$.
} 
\label{table:f_0508} 
\centering 
\def\arraystretch{1.}
\resizebox{\textwidth}{!}{
\begin{tabular}{lccccc}
\toprule
                 Task Name &                      SDM-GAN($f=0$) &                SDM-GAN($f=0.05$) &                  SDM-GAN($f=0.5$) &                  SDM-GAN($f=0.8$) &                      SDM-GAN($f=1$) \\
\midrule
              maze2d-umaze &   41.1 $\pm$ {\footnotesize 33.2} &  22.5 $\pm$ {\footnotesize 29.6} &   36.1 $\pm$ {\footnotesize 28.4} &   39.3 $\pm$ {\footnotesize 20.8} &   89.8 $\pm$ {\footnotesize 49.2} \\
             maze2d-medium &   79.2 $\pm$ {\footnotesize 27.6} &  71.5 $\pm$ {\footnotesize 17.0} &  115.4 $\pm$ {\footnotesize 34.2} &   90.8 $\pm$ {\footnotesize 49.6} &   89.3 $\pm$ {\footnotesize 38.8} \\
              maze2d-large &   96.8 $\pm$ {\footnotesize 60.8} &  52.0 $\pm$ {\footnotesize 63.2} &  207.7 $\pm$ {\footnotesize 11.7} &  127.1 $\pm$ {\footnotesize 62.1} &  101.4 $\pm$ {\footnotesize 73.6} \\
        halfcheetah-medium &    39.7 $\pm$ {\footnotesize 0.6} &   39.8 $\pm$ {\footnotesize 1.3} &    42.5 $\pm$ {\footnotesize 0.5} &    43.1 $\pm$ {\footnotesize 0.3} &    43.9 $\pm$ {\footnotesize 0.1} \\
           walker2d-medium &   52.6 $\pm$ {\footnotesize 12.9} &   49.6 $\pm$ {\footnotesize 2.3} &    66.7 $\pm$ {\footnotesize 1.8} &   58.9 $\pm$ {\footnotesize 12.4} &   58.3 $\pm$ {\footnotesize 10.7} \\
             hopper-medium &   44.6 $\pm$ {\footnotesize 19.7} &  44.5 $\pm$ {\footnotesize 32.0} &   62.8 $\pm$ {\footnotesize 14.3} &   68.2 $\pm$ {\footnotesize 22.3} &    54.5 $\pm$ {\footnotesize 6.2} \\
 halfcheetah-medium-replay &    41.1 $\pm$ {\footnotesize 0.4} &   41.3 $\pm$ {\footnotesize 0.5} &    41.7 $\pm$ {\footnotesize 0.4} &    40.5 $\pm$ {\footnotesize 0.8} &    36.6 $\pm$ {\footnotesize 2.7} \\
    walker2d-medium-replay &    13.2 $\pm$ {\footnotesize 2.7} &   14.0 $\pm$ {\footnotesize 2.2} &    20.3 $\pm$ {\footnotesize 4.0} &    18.4 $\pm$ {\footnotesize 3.6} &     4.0 $\pm$ {\footnotesize 2.0} \\
      hopper-medium-replay &    21.6 $\pm$ {\footnotesize 0.8} &   23.7 $\pm$ {\footnotesize 2.2} &    30.6 $\pm$ {\footnotesize 2.8} &    28.2 $\pm$ {\footnotesize 1.6} &    34.8 $\pm$ {\footnotesize 9.4} \\
 halfcheetah-medium-expert &   69.2 $\pm$ {\footnotesize 11.1} &   90.2 $\pm$ {\footnotesize 6.2} &    89.1 $\pm$ {\footnotesize 6.6} &    91.2 $\pm$ {\footnotesize 7.9} &   71.0 $\pm$ {\footnotesize 11.0} \\
    walker2d-medium-expert &   26.6 $\pm$ {\footnotesize 28.4} &  61.5 $\pm$ {\footnotesize 21.8} &    97.9 $\pm$ {\footnotesize 4.9} &   83.5 $\pm$ {\footnotesize 16.1} &   67.8 $\pm$ {\footnotesize 16.0} \\
      hopper-medium-expert &   33.0 $\pm$ {\footnotesize 25.4} &  32.8 $\pm$ {\footnotesize 17.3} &   104.5 $\pm$ {\footnotesize 5.4} &    99.9 $\pm$ {\footnotesize 9.7} &   80.7 $\pm$ {\footnotesize 10.6} \\
                 pen-human &    10.6 $\pm$ {\footnotesize 8.0} &  23.8 $\pm$ {\footnotesize 18.1} &    17.8 $\pm$ {\footnotesize 1.7} &   19.3 $\pm$ {\footnotesize 10.1} &     6.8 $\pm$ {\footnotesize 3.8} \\
                pen-cloned &   17.9 $\pm$ {\footnotesize 14.8} &  30.5 $\pm$ {\footnotesize 12.8} &    40.6 $\pm$ {\footnotesize 6.1} &    28.7 $\pm$ {\footnotesize 7.4} &   20.5 $\pm$ {\footnotesize 11.9} \\
                pen-expert &  116.3 $\pm$ {\footnotesize 21.0} &  125.0 $\pm$ {\footnotesize 7.1} &  135.8 $\pm$ {\footnotesize 11.7} &  114.8 $\pm$ {\footnotesize 19.2} &  113.5 $\pm$ {\footnotesize 14.0} \\
               door-expert &   35.8 $\pm$ {\footnotesize 43.2} &  61.1 $\pm$ {\footnotesize 42.8} &    93.5 $\pm$ {\footnotesize 6.7} &   68.0 $\pm$ {\footnotesize 26.2} &   70.3 $\pm$ {\footnotesize 14.2} \\
               \midrule
             Average Score &                              46.2 &                               49 &                              75.2 &                              63.7 &                                59 \\
\bottomrule
\end{tabular}
}
\end{table}
  
\begin{table}[H] 
\caption{Normalized returns for comparing our algorithm under different regularization coefficient $\alpha \in \cbr{2.5, 10, 25}$ and without regularization (No Reg.). 
The reported number are the means and standard deviation across three random seeds $\cbr{0,1,2}$.} 
\label{table:alpha_sweep} 
\centering 
\def\arraystretch{1.}
\resizebox{\textwidth}{!}{
\begin{tabular}{lcccc}
\toprule
                 Task Name &             SDM-GAN($\alpha=2.5$) &              SDM-GAN($\alpha=10$) &              SDM-GAN($\alpha=25$) &                 SDM-GAN(No Reg.) \\
\midrule
              maze2d-umaze &   37.2 $\pm$ {\footnotesize 30.8} &   36.1 $\pm$ {\footnotesize 28.4} &   66.1 $\pm$ {\footnotesize 24.3} &  35.0 $\pm$ {\footnotesize 64.6} \\
             maze2d-medium &   95.3 $\pm$ {\footnotesize 39.1} &  115.4 $\pm$ {\footnotesize 34.2} &   93.0 $\pm$ {\footnotesize 45.2} &  65.8 $\pm$ {\footnotesize 53.1} \\
              maze2d-large &   103.7 $\pm$ {\footnotesize 6.6} &  207.7 $\pm$ {\footnotesize 11.7} &  162.7 $\pm$ {\footnotesize 46.1} &   -1.4 $\pm$ {\footnotesize 1.4} \\
        halfcheetah-medium &    42.2 $\pm$ {\footnotesize 0.6} &    42.5 $\pm$ {\footnotesize 0.5} &    43.0 $\pm$ {\footnotesize 0.5} &   52.6 $\pm$ {\footnotesize 2.9} \\
           walker2d-medium &    58.7 $\pm$ {\footnotesize 7.5} &    66.7 $\pm$ {\footnotesize 1.8} &    58.7 $\pm$ {\footnotesize 9.0} &    1.2 $\pm$ {\footnotesize 2.4} \\
             hopper-medium &   46.8 $\pm$ {\footnotesize 23.0} &   62.8 $\pm$ {\footnotesize 14.3} &   44.4 $\pm$ {\footnotesize 14.6} &    1.5 $\pm$ {\footnotesize 0.7} \\
 halfcheetah-medium-replay &    39.6 $\pm$ {\footnotesize 0.3} &    41.7 $\pm$ {\footnotesize 0.4} &    43.0 $\pm$ {\footnotesize 0.4} &   56.3 $\pm$ {\footnotesize 2.4} \\
    walker2d-medium-replay &    15.9 $\pm$ {\footnotesize 3.8} &    20.3 $\pm$ {\footnotesize 4.0} &    22.7 $\pm$ {\footnotesize 0.8} &  16.3 $\pm$ {\footnotesize 12.5} \\
      hopper-medium-replay &    28.4 $\pm$ {\footnotesize 3.0} &    30.6 $\pm$ {\footnotesize 2.8} &    30.9 $\pm$ {\footnotesize 3.1} &  44.6 $\pm$ {\footnotesize 38.4} \\
 halfcheetah-medium-expert &   62.2 $\pm$ {\footnotesize 10.5} &    89.1 $\pm$ {\footnotesize 6.6} &    89.6 $\pm$ {\footnotesize 7.8} &  41.1 $\pm$ {\footnotesize 10.0} \\
    walker2d-medium-expert &   69.4 $\pm$ {\footnotesize 39.4} &    97.9 $\pm$ {\footnotesize 4.9} &   90.9 $\pm$ {\footnotesize 12.4} &  25.4 $\pm$ {\footnotesize 16.4} \\
      hopper-medium-expert &   96.1 $\pm$ {\footnotesize 15.9} &   104.5 $\pm$ {\footnotesize 5.4} &   79.9 $\pm$ {\footnotesize 31.3} &  32.9 $\pm$ {\footnotesize 13.4} \\
                 pen-human &   17.1 $\pm$ {\footnotesize 11.4} &    17.8 $\pm$ {\footnotesize 1.7} &   23.3 $\pm$ {\footnotesize 12.8} &   -2.7 $\pm$ {\footnotesize 1.6} \\
                pen-cloned &   29.8 $\pm$ {\footnotesize 14.6} &    40.6 $\pm$ {\footnotesize 6.1} &   34.9 $\pm$ {\footnotesize 19.5} &   -0.3 $\pm$ {\footnotesize 4.5} \\
                pen-expert &  115.4 $\pm$ {\footnotesize 14.9} &  135.8 $\pm$ {\footnotesize 11.7} &  114.2 $\pm$ {\footnotesize 18.6} &   -2.6 $\pm$ {\footnotesize 1.8} \\
               door-expert &   62.0 $\pm$ {\footnotesize 34.4} &    93.5 $\pm$ {\footnotesize 6.7} &   81.4 $\pm$ {\footnotesize 28.2} &   -0.2 $\pm$ {\footnotesize 0.1} \\
               \midrule
             Average Score &                              57.5 &                              75.2 &                              67.4 &                             22.8 \\
\bottomrule
\end{tabular}
}
\end{table}

%%%%%%%%%%%%%%%%%%%%%%%%%%%%%%%%%%%%%%%%%%%%%%%%%%%%%%%%%%%%%%%%%%%%%%%% 

%%%%%%%%%%%%%%%%%%%%%%%%%%%%%%%%%%%%%%%%%%
\section{Theoretical Analysis and Proofs} \label{sec:proofs}
\subsection{Preliminary on Average Reward Markov Decision Process}\label{sec:intro_average_case}
Our proof largely relies on the preliminary knowledge on average (undiscounted) reward Markov decision process \citep{markovdp1994}, so we will briefly introduce the general background on the undiscounted reward infinite-horizon RL.

\paragraph{Assumptions on Markov Decision Process (MDP).}
Follow the classic RL settings \citep{markovdp1994}, we model RL problem as as a Markov Decision Process, which consists of $\gM = \br{\sS, \sA, P, r}$, where $\sS$ denots the state space, $\sA$ the action space, $P(s^\prime | s,a)$ the  transition probability, which we assume to be time independent; $r(s,a)$ the intermediate reward function, which we assume to be deterministic and bounded ($|r(s,a)| \leq r_{\max}$).

Consider a time-stationary, Markov policy $\pi$, which specifies a distribution of actions given states,
and $\pi(a|s)$ denotes the probability of selecting $a$ given $s$.
The average reward of the policy $\pi$ is defined as 
\begin{align}
    \eta^\pi (s):= \lim_{T\to \infty}\E_{\pi, P}\left[\frac{1}{T+1}\sum_{t=0}^{T}r_t \mid s_0 = s\right]\,,\label{equ:def_avg_reward}
\end{align}
where the expectation $\E_{\pi, P}$ is with respect to the distribution of the trajectories where the states evolve according to $P$ and the actions are chosen by $\pi(a|s)$. 
When the states of the MDP is finite and the reward is bounded, the limit in Eq.~\eqref{equ:def_avg_reward} always exists \citep{markovdp1994}.
 The policy induces a Markov chain of states with the transition as $P^\pi(s^\prime |s)= \sum_{a}\pi(a|s)P(s^\prime | s,a)$. 
When $P^\pi$ is irreducible, it can be shown  that the stationary distribution of $P^\pi$ exists and is unique (denoted by $d_{\pi}$), and the average reward, $\eta^\pi(s)$ in Eq.~\eqref{equ:def_avg_reward} is independent of the initial state $s$ \citep[\textit{e.g.},][]{markovdp1994}, thus we have 
\begin{align}
    \eta^\pi = \eta^\pi(s)=\sum_{s,a}d_{\pi}(s)\pi(a|s)r(s,a)\,.
\end{align}
Also $d_{\pi}(s,a)=d_{\pi}(s)\pi(a|s)$ is the stationary state-action distribution induced by policy $\pi$ under the transition $P(s^\prime | s,a)$, which satisfy
\begin{align}
    d_{\pi}(s^\prime, a^\prime)= \sum_{s,a}\pi(a^\prime|s^\prime)P(s^\prime|s,a)d_{\pi}(s,a),~~~\forall (s^\prime, a^\prime)\in \sS \times \sA \,.\label{equ:stationary_bellman}
\end{align}
From previous assumption we know the fixed point solution of Eq.~\eqref{equ:stationary_bellman} exists and is unique. 
If we introduce a discriminative test function $g(s,a)\in \cG$, with $\cG:=\{g: \sS\times \sA\to \mathbb{R}, \norm{g}_{\infty}\leq G_{\max}\}$ as previous off-policy evaluation literature \citep[e.g.,][]{breakingcurse2018,gendice2020,tang2019doubly}, we can show 
\begin{align}
    \E_{(s,a)\sim d_{\pi}}\left[g(s,a)\right] = \E_{(s,a)\sim d_{\pi}, s^\prime\sim P(\cdot | s,a),a^\prime\sim \pi(\cdot | s^\prime)}\left[g(s^\prime, a^\prime)\right]\,,~~~\forall g\in \cG\,. \label{equ:test_eq}
\end{align}

\paragraph{Differential Value Function.} Unlike the discounted case where we define the value function $Q^\pi(s,a)$ as the expected total discounted reward when the initial state action pair is $(s,a)$: $Q^\pi(s,a):=\E_{\pi, P}\left[\sum_{t=0}^{\infty}\gamma^tr_t |s_0=s, a_0=a\right] $, in the average case, $Q^\pi(s,a)$ represents the \textit{expected total difference} between the reward and the average reward $\eta^\pi$
under the policy $\pi$, which is called \textit{differential (state-action) value function}\footnote{It is also referred as \textit{relative} value function in some RL literature.} in \citet{rlintro2018}:
\begin{align*}
    Q^\pi(s,a):= \E_{\pi, P}\left[\sum_{t=0}^{\infty}\left(r_t - \eta^\pi\right) ~\mid~s_0=s, a_0=a\right]\,.
\end{align*}

The differential value function $Q^\pi(s,a)$ and the average reward $\eta^\pi$ are closely related via the Bellman equation:
\begin{align}
    Q^\pi(s,a) = r(s,a) + \E_{s^\prime \sim P(\cdot | s,a), a^\prime \sim \pi(\cdot|s^\prime)}\left[Q^\pi(s^\prime, a^\prime)\right] - \eta^\pi\,,~~~\forall (s,a)\in \sS \times \sA \,.
\end{align}

For more details of the average reward RL, we refer the reader to classical RL books \citep[{\it e.g.},][]{markovdp1994,rlintro2018}.

\subsection{Proof of Theorem \ref{prop:ipm1}} \label{sec:proof_31}
\begin{theorem*}[Restatement of Theorem \ref{prop:ipm1}]
Suppose the distance metric $D$ is the integral probability metrics (IPM) {\it w.r.t} some function class $\mathcal{G}=\{g: \sS \times \sA \to \mathbb{R}\}$, defined as 
\begin{equation*}
    D_{\cG}\br{\dbtrue, d_{\pi}^{\widehat P}}:= \sup_{g\in \mathcal{G}}\left\vert\E_{\dbtrue}\left[g(s,a)\right] - \E_{d_{\pi}^{\widehat P}}\left[g(s,a)\right]\right\vert\,,
\end{equation*}
then there exists a function class $\gG$ such that we can rewrite $D_{\cG}(\dbtrue, d_{\pi}^{\widehat P})$ as:
\begin{equation*} 
D_{\cG}(\dbtrue, d_{\pi}^{\widehat P})= \mathcal{R}_{\cF}(\dbtrue, \pi,\widehat{P})
    :=\sup_{f\in \cF} \bigg\vert\E_{(s,a)\sim \dbtrue  }[f(s,a)]-\E_{{s^\prime \sim \widehat P(\cdot | s,a), a^\prime \sim \pi(\cdot | s^\prime)}\atop {(s,a)\sim \dbtrue }}\left[f(s^\prime, a^\prime)\right]\bigg\vert\,, 
\end{equation*}
where 
\begin{equation*}
\cF:=\Bigg\{f:~ f(s,a)=\E_{\pi, \widehat{P}}\left[\sum_{t=0}^{\infty}\left(g(s_t,a_t) - \hat{\eta}^\pi\right)~|~{{a_0=a}\atop{s_0=s}}\right]\, \hat{\eta}^\pi =\lim_{T\to \infty}\E_{\pi, \widehat{P}}\left[\frac{1}{T+1}\sum_{t=0}^{T}g(s_t, a_t)\right], g\in \cG\Bigg\}\,.
\end{equation*}
\end{theorem*}

\begin{proof}
Based on the preliminary introduction in Section.~\ref{sec:intro_average_case}, we introduce a new MDP $\widehat{\gM}=(\sS, \sA,\widehat P, g) $, where $\widehat P$ is the learned dynamics optimized by maximum likelihood estimation, and $g\in\mathcal{G}$, with $\cG :=\left\{g: \sS \times \sA \to \mathbb{R}, \norm{g(s,a)}_{\infty} \leq G_{\max}\right\}$.

Suppose the policy $\pi$ is a time-stationary, Markov policy, and $\widehat{P}^\pi(s^\prime|s):=\sum_{a}\widehat{P}(s^\prime|s,a)\pi(a|s)$ is finite state and ergodic.
$d_{\pi}^{\widehat P}(s,a)$ is the stationary state-action distribution induced by policy $\pi$ under the learned transition $\widehat P$.
Then we have the definition of the average reward $\widehat{\eta}^\pi$ and differential value function $\widehat{Q}^\pi (s,a)$ for the MDP $\widehat \gM$:
\begin{align}
    &\widehat{\eta}^\pi = \lim_{T\to \infty}\E_{\pi, \widehat P}\left[\frac{1}{T+1}\sum_{t=0}^{T} g(s_t, a_t)\right]=\E_{(s,a)\sim d_{\pi}^{\widehat P}}[g(s,a)]\,, \label{equ:avg_def}\\
     &\wQ^\pi(s,a):= \E_{\pi, \wP}\left[\sum_{t=0}^{\infty}\left(g(s_t, a_t) - \widehat \eta^\pi\right) ~\mid~s_0=s, a_0=a\right]\,, \label{equ:diff_def}\\
    &\wQ^\pi(s,a) = g(s,a) + \E_{s^\prime \sim \widehat{P}(\cdot | s,a), a^\prime \sim \pi(\cdot|s^\prime)}\left[\wQ^\pi(s^\prime, a^\prime)\right] - {\widehat{\eta}}^\pi\,,~~~~~~~~~~~~\forall ~(s,a)\in \sS \times \sA \,. \label{equ:new_bellman}
\end{align}

We next introduce a new function class $\mathcal{F}$ such that for each function $g \in \cG$, we have a corresponding function $f\in \cF$, which can represent the differential value function $\wQ^\pi$ given the reward function $g$ and the dynamics $\widehat P$:
\begin{align}
     \cF:=\Bigg\{f:~f(s,a)=\E_{\pi, \widehat{P}}\left[\sum_{t=0}^{\infty}\left(g(s_t,a_t) - \widehat{\eta}^\pi\right)~|~{{a_0=a}\atop{s_0=s}}\right]\,,
    \widehat{\eta}^\pi =\lim_{T\to \infty}\E_{\pi, \widehat{P}}\left[\frac{1}{T+1}\sum_{t=0}^{T}g(s_t, a_t)\right], g\in \cG\Bigg\}\,. \label{equ:def_F}
\end{align}
Moreover,  under classical regularity conditions made in Section \ref{sec:intro_average_case}, if we fix the model dynamics $\widehat P$, for each reward function $g(s,a)$, we have a unique one-to-one corresponding function $f\in \cF$ because the differential value function is the unique solution of the corresponding Bellman equation.
Besides, since the reward function is bounded, we assume $\sup_{f\in\mathcal{F}}\norm{f}_{\infty} \leq F_{\max}$, such that $\mathcal{F}$ is a bounded function class \citep{liao2020batch}.
%\st{@yihao: citation for this assumption?}. 
Then we have 
\begin{align*}
     D_{\cG}\br{\dbtrue, d_{\pi}^{\widehat P}} &= \sup_{g\in \mathcal{G}}\bigg\vert\E_{\dbtrue}\left[g(s,a)\right] - \underbrace{\E_{d_{\pi}^{\widehat P}}\left[g(s,a)\right]}_{\widehat{\eta}^\pi}\bigg\vert \\
     &=\sup_{f\in \cF}\left\vert \E_{(s,a)\sim \dbtrue}\left[f(s,a) - \E_{s^\prime \sim \widehat{P}(\cdot | s,a), a^\prime \sim \pi(\cdot|s^\prime)}\left[f(s^\prime, a^\prime)\right] + {\widehat{\eta}}^\pi \right] - \widehat{\eta}^\pi\right\vert \\
     &= \sup_{f\in \cF}\left\vert \E_{(s,a)\sim \dbtrue}\left[f(s,a)\right] - \E_{{s^\prime \sim \widehat{P}(\cdot | s,a), a^\prime \sim \pi(\cdot|s^\prime)}\atop{(s,a)\sim \dbtrue}}\left[f(s^\prime, a^\prime)\right]  \right\vert  \\
     &:=\mathcal{R}_{\cF}(\dbtrue, \pi,\widehat{P}) \,.
\end{align*}
If we choose the function class $\cF$ to be a parameterized neural network, then it should be rich enough to contain the true differential value function $\wQ^\pi(s,a)$ induced by dynamics $\widehat P$ and the deterministic reward function $g(s,a)$, because in practice we usually use a neural network to learn the value function $\wQ^\pi(s,a)$ .
\end{proof}

\subsection{Proof of Theorem \ref{thm:model_error}} \label{sec:proof_thm_32}
Before we prove Theorem \ref{thm:model_error}, we introduce the following lemma that will be used in the latter proof.

\begin{lemma} \label{thm:lemma_exchange_sup_exp}
Denote $P$ as an arbitrary probability measure, $\mathcal{Y}$ an arbitrary set.
Let $X \sim P, y \in \mathcal{Y}$, we have
\begin{equation*}
    \sup_{y \in \mathcal{Y}} \mathbb{E}_{X \sim P} [f(X, y)] \leq \mathbb{E}_{X \sim P}\left[\sup_{y \in \mathcal{Y}} f(X,y) \right].
\end{equation*}
\end{lemma}
\begin{proof}
$\forall \, y \in \mathcal{Y}, f(x,y) \leq \sup_{ y \in \mathcal{Y}} f(x,y), \forall\, x \implies \forall \, y \in \mathcal{Y}, \mathbb{E}_{X \sim P}\left[f(X,y) \right] \leq \mathbb{E}_{X \sim P}\left[\sup_{y \in \mathcal{Y}} f(X,y) \right]$. 
Taking $\sup$ on the left-hand-side, we have $\sup_{ y \in \mathcal{Y}} \mathbb{E}_{X \sim P}\left[f(X,y) \right] \leq \mathbb{E}_{X \sim P}\left[\sup_{y \in \mathcal{Y}} f(X,y) \right]$.
\end{proof}

\begin{theorem*}[Restatement of  Theorem \ref{thm:model_error}]
Let $P^{*}$ be the true dynamics, $D_\gG$ the IPM with the same $\gG$ as Theorem~\ref{prop:ipm1} and $\widehat{P}$ the estimated dynamics.
There exist a constant $C \geq 0$ and a function class $\Psi =\{\psi: \sS \times \sA \to \mathbb{R}\}$ such that for any policy $\pi$,
\begin{equation*}
    D_{\cG}(d_{\pi}^{\widehat P}, d_{\pi}^{P^*}) \leq C \cdot \E_{(s,a)\sim \dbtrue}\left[\sqrt{\frac{1}{2}\mathrm{KL}\left(P^{*}(s^\prime |s,a) \| \widehat{P}(s^\prime | s,a) \right)}\right] 
    + \mathcal{R}_{\Psi}(\dbtrue, \pi)\,, 
\end{equation*}
where 
\begin{equation*}
 \mathcal{R}_{\Psi}(\dbtrue, \pi):=\sup_{\psi \in \Psi}\left\vert \E_{(s,a)\sim \dbtrue}[\psi(s,a)] - \E_{s\sim \dbtrue\atop{a\sim \pi(\cdot|s)}}[\psi(s,a)]\right\vert\,.
\end{equation*}
\end{theorem*}

\begin{proof}
Still, we will use the proof technique when we prove Proposition \ref{prop:ipm1}. 

We construct a new MDP $\widehat{\gM}=(\sS, \sA,\widehat P, g)$, where $\widehat P$ is the learned dynamics and $g\in\mathcal{G}$ is an intermediate reward function, with $\cG :=\left\{g: \sS \times \sA \to \mathbb{R}, \norm{g(s,a) }_{\infty} \leq G_{\max}\right\}$. 
The definition of average reward $\widehat{\eta}^\pi$, stationary distribution $d_{\pi}^{\widehat P } (s,a)$, differential value function $\widehat{Q}^\pi (s,a)$, and the function class $\mathcal{F}$ are the same as in Section \ref{sec:proof_31}.
if the distance metric is the integral probability metric (IPM) defined on test function class $\cG$, then we have
\begin{align*}
    D_{\cG}\br{d_{\pi}^{\wP}, d_\pi^\trueP}& := \sup_{g\in \cG}\bigg\vert {\E_{d_\pi^\wP}[g(s,a)}  - \E_{d_\pi^\trueP}[g(s,a)] \bigg\vert\,.
   % &= \sup_{g\in \cG} \left\vert \widehat \eta^\pi - \E_{} \right\vert
\end{align*}
Before we take the supremum over the function class $\cG$, we first rewrite the term inside the $\sup$:
\begin{align*}
    \bigg\vert {\E_{d_\pi^\wP}[g(s,a)}  - \E_{d_\pi^\trueP}[g(s,a)] \bigg\vert &=
    \bigg\vert \underbrace{\E_{d_\pi^\wP}[g(s,a)]}_{\widehat \eta^\pi}  - \E_{d_\pi^\trueP}[g(s,a)] \bigg\vert \\
    &= \left\vert \widehat \eta^\pi -\E_{d_\pi^\trueP}\left[\wQ^\pi (s,a)\right] +\E_{d_\pi^\trueP}\left[\wQ^\pi (s,a)\right] -\E_{d_\pi^\trueP}[g(s,a)] \right\vert ~~~~//~~\textcolor{red}{\wQ^\pi(s,a)~\textit{defined~in}~\eqref{equ:diff_def}}\\
    &=\left\vert  -\E_{d_\pi^\trueP}\left[\wQ^\pi (s,a)\right]  + \E_{d_\pi^\trueP}\left[\wQ^\pi (s,a) - g(s,a) + \widehat \eta^\pi\right] \right\vert \\
    &= \left\vert  -\E_{d_\pi^\trueP}\left[\wQ^\pi (s,a)\right]  + \E_{(s,a)\sim d_\pi^\trueP}\left[\E_{s^\prime \sim \widehat{P}(\cdot | s,a), a^\prime \sim \pi(\cdot|s^\prime)}\left[\wQ^\pi(s^\prime, a^\prime)\right]\right] \right\vert ~~~~//~~\textcolor{red}{use~\eqref{equ:new_bellman}} \\
    &=  \Bigg\vert  -\underbrace{\E_{(s,a)\sim d_\pi^\trueP}\left[\E_{{a^\prime \sim \pi(\cdot|s^\prime)}\atop{\textcolor{blue}{s^\prime \sim \trueP(\cdot | s,a)}}}\left[\wQ^\pi(s^\prime, a^\prime)\right]\right]}_{\textit{use}~\eqref{equ:test_eq}}  + \E_{(s,a)\sim d_\pi^\trueP}\left[\E_{{a^\prime \sim \pi(\cdot|s^\prime)}\atop{\textcolor{red}{s^\prime \sim \widehat{P}(\cdot | s,a)}}}\left[\wQ^\pi(s^\prime, a^\prime)\right]\right] \Bigg\vert \\
    &=\Bigg\vert \E_{(s,a)\sim d_{\pi}^\trueP}\Bigg[\underbrace{\E_{{a^\prime \sim \pi(\cdot|s^\prime)}\atop{\textcolor{red}{s^\prime \sim \widehat{P}(\cdot | s,a)}}}\left[\wQ^\pi(s^\prime, a^\prime)\right] - \E_{{a^\prime \sim \pi(\cdot|s^\prime)}\atop{\textcolor{blue}{s^\prime \sim \trueP(\cdot | s,a)}}}\left[\wQ^\pi(s^\prime, a^\prime)\right]}_{\defeq~ \mathcal{E}(s,a; \widehat P)}\Bigg] \Bigg\vert \\
    &= \left\vert\E_{(s,a)\sim d_{\pi}^\trueP}\left[\mathcal{E}(s,a; \widehat P)\right]\right\vert \,.
\end{align*}

Similarly, for each reward function $g(s,a)\in \cG$, we have a unique one-to-to corresponding function $f\in \cF$, with the function class $\cF$ defined in Eq.~\eqref{equ:def_F}.
Then we can change from taking $\sup$ {\it w.r.t.} the function class $\cG$ to {\it w.r.t.} the function class $\cF$:
\begin{align*}
     D_{\cG}\br{d_{\pi}^{\wP}, d_\pi^\trueP} &= \sup_{f\in \mathcal{F}}\left\vert\E_{(s,a)\sim d_{\pi}^\trueP}\left[\mathcal{E}(s,a; \widehat P)\right]\right\vert \\
     &=\sup_{f \in \cF}\left\vert \int \int \mathcal{E}(s,a; \widehat P)\left(d_{\pi}^\trueP(s,a) - \dbtrue(s,a) + \dbtrue(s,a) \right)dads \right\vert  \\
     &\leq \underbrace{\sup_{f \in \cF}\left\vert \int \int \mathcal{E}(s,a; \widehat P) \dbtrue(s,a)dads \right\vert}_{\defeq ~\circled{1}} +\underbrace{\sup_{f\in \mathcal{F}}\left\vert \int\int \gE(s,a;\widehat P)\left(d_{\pi}^\trueP(s,a) - \dbtrue(s,a)\right)dads\right\vert}_{\defeq ~\circled{2}}\,.
\end{align*}
Next we are going to analyze \circled{1} and \circled{2} separately.

\paragraph{Analysis on \circled{1}.} 
Based on the definition of function class $\cF$, we introduce a new function class $\mathcal{V}$, which is defined as 
\begin{align*}
    \mathcal{V}\triangleq \left\{v:~ v(s)=\E_{a\sim \pi(\cdot | s)}[f(s,a)], f \in \cF \right\}\,,
\end{align*}
where $v(s)$ can be viewed as the state value function under the MDP $\widehat{\gM}=(\sS, \sA,\widehat P, g)$, and $f(s,a)$ is the state-action value function under the definition, 
and $\norm{v}_{\infty} \leq F_{\max}$ since $\norm{f}_{\infty} \leq F_{\max}$.
Recall that the total variation distance is a special instance of the integral probability metrics $D_\gG$ when the function class $\gG = \cbr{g: \norm{g}_\infty \leq 1}$. \citep{ipm2009,mmdgangp2018}.
Thus we have
\begin{equation} \label{eq:upper_bound_of_1}
\begin{split}
    \circled{1}&= \sup_{f \in \cF}\left\vert \int \int \mathcal{E}(s,a; \widehat P) \dbtrue(s,a)dads \right\vert \\
    &=\sup_{f \in \cF} \left\vert \E_{(s,a)\sim \dbtrue}\left[{\E_{{a^\prime \sim \pi(\cdot|s^\prime)}\atop{\textcolor{red}{s^\prime \sim \widehat{P}(\cdot | s,a)}}}\left[f(s^\prime, a^\prime)\right] - \E_{{a^\prime \sim \pi(\cdot|s^\prime)}\atop{\textcolor{blue}{s^\prime \sim \trueP(\cdot | s,a)}}}\left[f(s^\prime, a^\prime)\right]}\right]  \right\vert \\
    &= \sup_{v \in \mathcal{V}} \left\vert \E_{(s,a)\sim \dbtrue}\left[{\E_{\textcolor{red}{s^\prime \sim \widehat{P}(\cdot | s,a)}}\left[v(s^\prime)\right] - \E_{{\textcolor{blue}{s^\prime \sim \trueP(\cdot | s,a)}}}\left[v(s^\prime)\right]}\right]  \right\vert \\
   &\leq \E_{(s,a)\sim \dbtrue}\left[\sup_{v\in \mathcal{V}}\left\vert {\E_{\textcolor{red}{s^\prime \sim \widehat{P}(\cdot | s,a)}}\left[v(s^\prime)\right] - \E_{{\textcolor{blue}{s^\prime \sim \trueP(\cdot | s,a)}}}\left[v(s^\prime)\right]}\right\vert\right] \\
   &= F_{\mathrm{max}}  \E_{(s,a)\sim \dbtrue} \left[ \sup_{v\in\mathcal{\mathcal{V}}}\left\vert \mathbb{E}_{\textcolor{red}{s^\prime \sim \widehat{P}(\cdot | s,a)}}\left[\frac{v(s^{\prime})}{F_{\mathrm{max}}}\right]-\mathbb{E}_{\textcolor{blue}{s^\prime \sim \trueP(\cdot | s,a)}}\left[\frac{v(s^{\prime})}{F_{\mathrm{max}}}\right]
 \right\vert \right] \\
 &= F_{\mathrm{max}} \mathbb{E}_{(s,a)\sim \dbtrue} \left[ \mathrm{TV}\br{P^*(\cdot \given s,a) \| \widehat P(\cdot \given s,a)} \right] \\
 &\leq F_{\mathrm{max}} \mathbb{E}_{(s,a)\sim \dbtrue} \left[ \sqrt{{1 \over 2} \mathrm{KL}\br{P^*(\cdot \given s,a) \| \widehat P(\cdot \given s,a)}} \right] \,.
 \end{split}
\end{equation}
since $\sup_{v \in \mathcal{V}}\norm{\frac{v}{F_{\mathrm{max}}}}_\infty \leq 1$, where we use Lemma~\ref{thm:lemma_exchange_sup_exp} to exchange order between $\sup$ and $\E[\cdot]$, and the last bound is from the Pinsker's inequality. 

\paragraph{Analysis on \circled{2}.}
From previous assumption we know $\sup_{v \in \mathcal{V}} \|v\|_{\infty} \leq F_{\max}$, then we can bound $\gE(s,a;\widehat P)$ by 
\begin{align*}
    \abs{\gE(s,a;\widehat P)}&= \abs{\E_{s^\prime \sim \wP(\cdot | s,a)}[v(s^\prime)] - \E_{s^\prime \sim \trueP(\cdot | s,a)}[v(s^\prime)]} \\
    & \leq \left\vert \E_{s^\prime \sim \wP(\cdot | s,a)}[v(s^\prime)]\right\vert + \left\vert \E_{s^\prime \sim \trueP(\cdot | s,a)}[v(s^\prime)]\right\vert\\
    &\leq 2F_{\max}\,.
\end{align*}

Denote the function class for $\gE(s,a;\widehat P)$ as $\gB:=\cbr{\gE(s,a;\widehat P), \widehat P \in \gP}$, which is a bounded function class since $\abs{\gE(s,a ; \wP)} \leq 2F_{\max}, \forall \, \gE(s,a;\widehat P) \in \gB$.
Since $\gE(s,a;\widehat P)$ is linear {\it w.r.t.} $f$, we have 
\begin{align*}
    \circled{2} &=\sup_{f\in \mathcal{F}}\left\vert \int\int \gE(s,a;\widehat P)\left(d_{\pi}^\trueP(s,a) - \dbtrue(s,a)\right)dads\right\vert \\
    &= \sup_{\gE(s,a;\widehat P)\in \gB} \left\vert\E_{\dbtrue}[\gE(s,a;\widehat P)] -\E_{d_{\pi}^\trueP}[\gE(s,a;\widehat P)] \right\vert\,.
\end{align*}
Still, we will use change of the variable trick as we use in the proof of Proposition \ref{prop:ipm1} to make \circled{2} tractable.
Define function class $\Psi$:
\begin{equation*} 
    \resizebox{\textwidth}{!}{%
$
\begin{aligned}
     \Psi:=\Bigg\{\psi:~\psi(s,a)=\E_{\pi, \trueP}\left[\sum_{t=0}^{\infty}\left(\gE(s_t,a_t;\widehat P) - {\eta}^\pi\right)~|~{{a_0=a}\atop{s_0=s}}\right]\,,
    {\eta}^\pi =\lim_{T\to \infty}\E_{\pi, {P}^{*}}\left[\frac{1}{T+1}\sum_{t=0}^{T}\gE(s_t,a_t;\widehat P)\right], \gE(s,a;\widehat P) \in \gB\Bigg\}\,. %\label{equ:def_Psi}
\end{aligned}
    $%    
}
\end{equation*}
Under classical regularity conditions made in Section \ref{sec:intro_average_case}, if we fix the true dynamics $\trueP$, for each reward function $\gE(s,a;\widehat P) \in \gB$, we have a unique one-to-one corresponding function $\psi\in \Psi$, because the differential value function is the unique solution of the corresponding Bellman equation.
Besides, since the reward function is bounded, we assume $\sup_{\psi\in\Psi}\norm{\psi}_{\infty} \leq \Psi_{\max}$, such that $\Psi$ is a bounded function class \citep{liao2020batch}. So we can rewrite \circled{2} as 
\begin{align}
    \circled{2}&= \sup_{\gE(s,a;\widehat P)\in \gB} \left\vert\E_{\dbtrue}[\gE(s,a;\widehat P)] -\E_{d_{\pi}^\trueP}[\gE(s,a;\widehat P)] \right\vert \notag\\
    &= \sup_{\psi\in \Psi} \left\vert\E_{(s,a)\sim \dbtrue}\left[\psi(s,a) - \E_{s^\prime \sim \trueP(\cdot | s,a), a^\prime \sim \pi(\cdot|s^\prime)}\left[\psi(s^\prime, a^\prime)\right] + \eta^\pi\right] -\eta^\pi \right\vert \notag\\
    &= \sup_{\psi\in \Psi} \left\vert\E_{(s,a)\sim \dbtrue}\left[\psi(s,a) \right]-\E_{{s^\prime \sim \trueP(\cdot | s,a), a^\prime \sim \pi(\cdot|s^\prime)}\atop{(s,a)\sim \dbtrue}}\left[\psi(s^\prime, a^\prime)\right] \right\vert \notag\\
    &= \sup_{\psi\in \Psi} \left\vert\E_{(s,a)\sim \dbtrue}\left[\psi(s,a) \right]-\E_{{s\sim \dbtrue}, a\sim \pi(\cdot | s)}\left[\psi(s,a)\right]\right\vert  \,.\label{equ:res_2}
\end{align}
The last equality comes from the fact that for offline datasets collected by sequential rollouts, the marginal distribution of $s$ and $s^\prime$ are the same.
In other words, if we randomly draw $s\sim d_{\pi}^\trueP(\cdot)$, $s$ will almost always be the “next state” of some other state in the dataset.

\paragraph{Combining analysis of \circled{1} and \circled{2}.}
Combing Eq.~\eqref{eq:upper_bound_of_1} and Eq.~\eqref{equ:res_2}, we have the final results

\begin{equation*}
\resizebox{\textwidth}{!}{%
$
    D_{\cG}(d_{\pi}^{\widehat P}, d_{\pi}^{P^*}) \leq C \cdot \E_{(s,a)\sim \dbtrue}\left[\sqrt{\frac{1}{2}\mathrm{KL}\left(P^{*}(s^\prime |s,a) \| \widehat{P}(s^\prime | s,a) \right)}\right] 
    +\sup_{\psi\in \Psi} \left\vert\E_{(s,a)\sim \dbtrue}\left[\psi(s,a) \right]-\E_{{s\sim \dbtrue}, a\sim \pi(\cdot | s)}\left[\psi(s,a)\right]\right\vert \,,
$%    
}
\end{equation*}

where $C$ is a constant, which may require tuning when implemented in practise.  
The proof of Theorem~\ref{thm:model_error} is completed.
\end{proof}

\begin{proposition}\label{eq:kl_mle}
An upper bound of \circled{1}, the last term of Eq.~\eqref{eq:upper_bound_of_1}, can be minimized via the maximum likelihood estimation for $\widehat P$.
\end{proposition}
\begin{proof}
Based on Eq.~\eqref{eq:upper_bound_of_1}, we would like to minimize the last equation {\it w.r.t.} $\widehat P$, and by Jensen's inequality, we have,
\begin{equation*}
    \begin{split}
        & \min_{\widehat P} F_{\mathrm{max}} \mathbb{E}_{(s,a)\sim \dbtrue} \left[ \sqrt{{1 \over 2} \mathrm{KL}\br{P^*(\cdot \given s,a) \| \widehat P(\cdot \given s,a)}} \right] \\
        \leq&  \min_{\widehat P}  F_{\mathrm{max}}  \sqrt{{1 \over 2} \mathbb{E}_{(s,a)\sim \dbtrue} \sbr{\mathrm{KL}\br{P^*(\cdot \given s,a) \| \widehat P(\cdot \given s,a)}}}  \\
        =&  \min_{\widehat P} F_{\mathrm{max}} \sqrt{{1\over 2} \mathbb{E}_{(s,a)\sim \dbtrue}\sbr{\int P^*(s' \given s,a)\br{\log P^*(s' \given s,a) - \log \widehat P(s' \given s,a)} ds'}} \\
        =&  \min_{\widehat P} F_{\mathrm{max}} \sqrt{{1\over 2} \cbr{\mathbb{E}_{(s,a)\sim \dbtrue}\sbr{\int P^*(s' \given s,a) \log P^*(s' \given s,a) ds'} - \mathbb{E}_{(s,a)\sim \dbtrue} \sbr{\int P^*(s' \given s,a)\log \widehat P(s' \given s,a) ds'}}}    
        \\ =&  F_{\mathrm{max}} \sqrt{ {1\over 2} \cbr{\mathbb{E}_{(s,a,s')\sim \dbtrue}\sbr{ \log P^*(s' \given s,a) } + \min_{\widehat P} \cbr{- \mathbb{E}_{(s,a,s')\sim \dbtrue} \sbr{ \log \widehat P(s' \given s,a)}}}},
    \end{split}
\end{equation*}
where both $F_{\mathrm{max}}$ and the first term inside $\sqrt{\cdot}$ are constants  {\it w.r.t.} $\widehat P$, and 
\begin{equation*}
    \arg \min_{\widehat P} \cbr{- \mathbb{E}_{(s,a,s')\sim \dbtrue} \sbr{ \log \widehat P(s' \given s,a)}} = \arg \max_{\widehat P}  \mathbb{E}_{(s,a,s')\sim \dbtrue} \sbr{ \log \widehat P(s' \given s,a)}
\end{equation*}
which is exactly the maximum-likelihood model-training objective Eq.~\eqref{eq:model_objective}.
\end{proof}

\subsection{Proof of Theorem \ref{col:sum_up}}
\begin{theorem*}[Restatement of Theorem \ref{col:sum_up}]
Combining Theorem \ref{prop:ipm1}, \ref{thm:model_error} and Proposition \ref{prop:kl_mle_sec_method}, and suppose the dynamics model $\widehat P$ is trained by the MLE objective (Eq.~\eqref{eq:model_objective}), then for the same $\gG$ and any policy $\pi$, up to some constants {\it w.r.t.} $\pi$ and $\widehat P$,
\begin{equation*}
 D_{\cG}(\dbtrue, \dpitrue) \leq \mathcal{R}_{\cF}(\dbtrue, \pi,\widehat{P}) + \mathcal{R}_{\Psi}(\dbtrue, \pi) + \mathcal{E}_{model},
\end{equation*}
where $\mathcal{E}_{model}$ is the error associated with the maximum likelihood training, and is independent of the policy $\pi$.
\end{theorem*}

\begin{proof}
From Eq.~\eqref{equ:reg_relax} we have
\begin{align}
    D_{\gG}(\dbtrue, \dpitrue) \leq \red{D_{\gG}(\dbtrue, d_{\pi}^{\widehat P})} + \textcolor{blue}{D_{\gG}(d_{\pi}^{\widehat P}, d_{\pi}^{P^*})}\,.
\end{align}
From Theorem~\ref{prop:ipm1}, we have
\begin{equation*} 
\red{D_{\cG}(\dbtrue, d_{\pi}^{\widehat P})} = \mathcal{R}_{\cF}(\dbtrue, \pi,\widehat{P})
    :=\sup_{f\in \cF} \bigg\vert\E_{(s,a)\sim \dbtrue  }[f(s,a)]-\E_{{s^\prime \sim \widehat P(\cdot | s,a), a^\prime \sim \pi(\cdot | s^\prime)}\atop {(s,a)\sim \dbtrue }}\left[f(s^\prime, a^\prime)\right]\bigg\vert\,, 
\end{equation*}

From Theorem~\ref{thm:model_error} we have,
\begin{equation*}
    \textcolor{blue}{D_{\cG}(d_{\pi}^{\widehat P}, d_{\pi}^{P^*})} \leq C \cdot \E_{(s,a)\sim \dbtrue}\left[\sqrt{\frac{1}{2}\mathrm{KL}\left(P^{*}(s^\prime |s,a) \| \widehat{P}(s^\prime | s,a) \right)}\right] 
    + \mathcal{R}_{\Psi}(\dbtrue, \pi)\,, 
\end{equation*}
where 
\begin{equation*}
 \mathcal{R}_{\Psi}(\dbtrue, \pi):=\sup_{\psi \in \Psi}\left\vert \E_{(s,a)\sim \dbtrue}[\psi(s,a)] - \E_{s\sim \dbtrue\atop{a\sim \pi(\cdot|s)}}[\psi(s,a)]\right\vert\,,
\end{equation*}
$C$ is some constant depending on the function class $\gG$ and the model error $\mathcal{E}_{model}$ is
\begin{equation*}
    C \cdot \mathcal{E}_{model} = \E_{(s,a)\sim \dbtrue}\left[\sqrt{\frac{1}{2}\mathrm{KL}\left(P^{*}(s^\prime |s,a) \| \widehat{P}(s^\prime | s,a) \right)}\right]
\end{equation*}

From Proposition~\ref{prop:kl_mle_sec_method}, we have an upper bound for $\mathcal{E}_{model}$,
\begin{equation*}\textstyle
    \mathcal{E}_{model} \leq \sqrt{ {1\over 2} \cbr{\mathbb{E}_{(s,a,s')\sim \dbtrue}\sbr{ \log P^*(s' \given s,a) } + \cbr{- \mathbb{E}_{(s,a,s')\sim \dbtrue} \sbr{ \log \widehat P(s' \given s,a)}}}}
\end{equation*}
where for a given true MDP $\gM$, $\mathbb{E}_{(s,a,s')\sim \dbtrue}\sbr{ \log P^*(s' \given s,a) }$ is a constant {\it w.r.t.} the policy $\pi$ and the learned dynamics $\widehat P$ and the upper bound is essentially determined by 
\begin{equation*}
    \mathcal{E}_{model} \leq O\br{ \sqrt{  - \mathbb{E}_{(s,a,s')\sim \dbtrue} \sbr{ \log \widehat P(s' \given s,a)}}}, 
\end{equation*}
where minimizing this upper bound amounts to the maximum likelihood training for $\widehat P$ and does not involve the policy $\pi$.

Putting everything together, we get the desired inequality.
\end{proof}

\section{Technical Details} \label{sec:exp_details}
\subsection{Details for the Toy Experiment} \label{sec:toy_details}
We conduct the toy experiment following  \citet{wgangp2017}. We set up the total sample size as $N_{\mathrm{total}}$ as $N_{\mathrm{total}}=100000$, radius $r$ as $4$, and standard deviation $\sigma$ for the Gaussian noise as $0.05$. The dataset is randomly split into a training  and testing set with the size of $5000$.
Details are outlined in Algorithm~\ref{alg:circle_dataset}, where $\odot$ denotes element-wise multiplication.

\begin{algorithm}[ht]
\caption{Constructing the Circle Dataset}
\begin{algorithmic}
\label{alg:circle_dataset}
    \STATE {\bfseries Input:} Total sample size $N_{\mathrm{total}}$, radius $r$.
    \STATE {\bfseries Output:} Generated dataset $\sD_{\mathrm{Circle}}$.
    \STATE Sample angles $\vtheta \sim \mathrm{Unif}(0, 2\pi)$, sample noise $\vepsilon \sim \gN(0, \sigma^2)$, $\vr \leftarrow r + \vepsilon$.
    \STATE Sample data point $(x,y)$ as $(x, y) = \br{\vr \odot \cos(\vtheta), \vr \odot \sin(\vtheta)}$.
\end{algorithmic}%}
\end{algorithm}
% In our experiment we use $N_{\mathrm{total}}=100000, r=4, \sigma = 0.05$, and the dataset is randomly split into a testset of size $5000$ and the rest as training set.

To interpret the circle dataset as an offline reinforcement learning task, we define $x$ as state and the corresponding $y$ as action. 
Note that in the behavior cloning task, the information of reward, next state, and the episodic termination is unnecessary. 

Hence, the generated dataset $\sD_{\mathrm{Circle}}$ can serve as an offline RL dataset which is applicable to train behavior cloning policies.
the circle dataset at here is able to  indicate multimodality in almost all the states, \textit{i.e.}, the conditional distribution $p\br{y \given x}$ is multi-modal in almost all $x$.
Therefore, a good policy is expected to capture all the action modes. 

To show the advantage of using a flexible implicit policy to approximate the behavior policy, we fit a conditional GAN \citep{cgan2014} (``CGAN"), Gaussian generator conditional GAN (``G-CGAN"), and a deterministic-generator conditional GAN (``D-CGAN").
All these CGAN variants are fitted using the conditional-distribution matching approach similar to \citet{brac2019}. 
Following \citet{wgangp2017} and \citet{gantutorial2017}, we use {\tt DIM} = $128$,  {\tt noise\_dim} = $2$, batch size $256$, critic iteration $5$, one-side label smoothing of $0.9$, and total training iterations $100000$. 
In Figure~\ref{fig:toy}, we show the kernel-density-estimate plots on the test set for each method using random seed $1$.
We have experimented on several other random seeds and the results do not significantly alter.
The detailed architecture of our generators and discriminator are shown as follows.

\begin{multicols}{2}
\begin{minipage}{0.45\textwidth}
Implicit Generator
\begin{verbatim}
Linear(state_dim+noise_dim, DIM),
ReLU(),
Linear(DIM, DIM),
ReLU(),
Linear(DIM, DIM),
ReLU(),
Linear(DIM, action_dim)
\end{verbatim}
\end{minipage}

\begin{minipage}{0.45\textwidth}
Deterministic Generator
\begin{verbatim}
Linear(state_dim, DIM),
ReLU(),
Linear(DIM, DIM),
ReLU(),
Linear(DIM, DIM),
ReLU(),
Linear(DIM, action_dim)
\end{verbatim}
\end{minipage}
\end{multicols}

\begin{multicols}{2}
\begin{minipage}{0.45\textwidth}
Gaussian Generator
\begin{verbatim}
Linear(state_dim, DIM),
ReLU(),
Linear(DIM, DIM),
ReLU(),
Linear(DIM, DIM),
ReLU()
mean=Linear(DIM, action_dim)
logstd=Linear(DIM,action_dim).clamp(-5,5)
\end{verbatim}
\end{minipage}

\begin{minipage}{0.45\textwidth}
Discriminator
\begin{verbatim}
Linear(state_dim+action_dim, DIM),
ReLU(),
Linear(DIM, DIM),
ReLU(),
Linear(DIM, DIM),
ReLU(),
Linear(DIM, 1),
Sigmoid()
\end{verbatim}
\end{minipage}
\end{multicols}

\subsection{Details for the Main Algorithm} \label{sec:algo_details}
As in prior model-based offline RL algorithms, our main algorithm can be divided into two parts: model training (Appendix~\ref{sec:algo_details_model}) and actor-critic training (Appendix~\ref{sec:algo_details_ac}).

\subsubsection{Model Training} \label{sec:algo_details_model}
In this paper, we assume no prior knowledge about the reward function and the termination condition. 
Hence we use neural network to approximate transition dynamics, reward function and the termination condition.
We follow prior work \citep{pets2018, mbpo2019, mopo2020} to construct our model as an ensemble of Gaussian probabilistic networks.
We use the same model architecture, ensemble size (=7), number of elite model (=5), train-test set split, optimizer setting, elite-model selection criterion, and sampling strategy as in \citet{mopo2020}.

As in \citet{mopo2020}, the input to our model is $\br{(s,a) - \mu_{(s,a)}} / \sigma_{(s,a)}$, where $\mu_{(s,a)}$ and $\sigma_{(s,a)}$ are respectively mean and standard deviation of $(s,a)$ in the offline dataset.
We follow \citet{morel2020} and \citet{bremen2021} to define the learning target as $\br{\br{r, \Delta s} - \mu_{\br{r, \Delta s}}} / \sigma_{\br{r, \Delta s}}$, where $\Delta s$ denote $s' - s$ ,  $\mu_{\br{r, \Delta s}}$ and $\sigma_{\br{r, \Delta s}}$ denote the mean and standard deviation of $\br{r, \Delta s}$ in the offline dataset.
As in prior model-based RL using Gaussian probabilistic ensemble \citep{pets2018,  mbpo2019, mopo2020, mabe2021, combo2021}, the output of our model is double-headed, respectively outputting the mean and log-standard-deviation of the normal distribution of the estimated output.
As in \citet{mopo2020}, the loss function is maximum likelihood, independently for each model in the ensemble.

We augment the mean-head of the output layer to further output the probability of the input satisfying the termination condition, which we model as a deterministic function of the input and is trained using weighted binary-cross entropy (WBCE) loss with class-weight inverse-proportional to the ratio between the number of terminated and non-terminated samples in the offline dataset.
This WBCE loss is added to the losses of the probabilistic networks in both model training and elite-model selection.
We use a cutoff value of $0.5$ for deciding termination.

\subsubsection{Actor-Critic Training} \label{sec:algo_details_ac}
Our training process for actor and critic can be divided into the follow parts.

\textbf{Generate synthetic data using dynamics model.} 
Perform $h$-step rollouts using the learned model $\widehat P$ and the current policy $\pi_\vphi$ by branching from the offline dataset $\denv$, and adding the generated data to a separate replay buffer $\dmodel$, as in \citet{mbpo2019} and \citet{mopo2020}.

For computational efficiency, we perform model rollout every {\tt rollout\_generation\_freq} iterations.

\textbf{Critic Training.} 
With hyperparameter $f \in [0,1]$, sample mini-batch $\gB$ from $\gD = f \cdot \denv + (1 - f)\cdot \dmodel$.

To smooth the discrete $\delta_{s'}$ transition kernel recorded in the offline dataset, we apply the state-smoothing trick in \citet{jointmatching2022} in calculating the value of the next state. Specifically, for each $s' \in \gB$ sample $N_B$ $\hat s$ with noise standard deviation $\sigma_B$ for state-smoothing via,
\begin{equation*}\textstyle
    \hat s = s' + \epsilon, \epsilon \sim \gN(\vzero, \sigma_B^2 \mI).
\end{equation*}
Sample $N_a$ corresponding actions $\hat a \sim \pi_{\vphi'}(\cdot \given \hat s)$ for each $\hat s$.
Calculate the estimated value of next state,
\begin{equation*}
         \widetilde{Q}\br{s, a} \triangleq \frac{1}{N_B \times N_a} \sum_{\br{\hat s,\hat a}} \sbr{c \min_{j=1,2}Q_{\vtheta'_j}\br{\hat s, \hat a} + (1-c) \max_{j=1,2}Q_{\vtheta'_j}\br{\hat s, \hat a}}.
\end{equation*}
Calculate the critic-learning target defined as
\begin{equation} \label{eq:critic_target}
    \widetilde{Q}\br{s, a} \leftarrow r(s,a).\mathrm{clamp}(r_{\mathrm{min}}, r_{\mathrm{max}}) + \gamma \cdot \widetilde{Q}\br{s, a} \cdot \br{\abs{\widetilde{Q}\br{s, a}} < 2000},
\end{equation}
where $r_{\mathrm{min}} = r_0 - 3 \sigma_r, r_{\mathrm{max}} = r_1 + 3\sigma_r$ for $r_0, r_1, \sigma_r$ the minimum, maximum and standard deviation of rewards in the offline dataset, similar to the reward-clapping strategy in \citet{wmopo2021}; $2000$ is some convenient number as an augmentation to the termination condition recorded in $\gB$.

Minimize the critic loss with respect to $\vtheta_j, j = 1,2$, over $\br{s, a} \in \gB$, with learning rate $\eta_\vtheta$,
\begin{equation*}\textstyle
    \arg\min_{\vtheta_j}\frac{1}{\abs{\gB}}\sum_{\br{s, a} \in \gB} \mathrm{Huber}\br{ Q_{\vtheta_j}\br{s, a} - \widetilde{Q}(s, a) },
\end{equation*}
where $\mathrm{Huber}(\cdot)$ denote the Huber-loss function, with threshold conveniently chosen as $500$.

In actual implemention, we choose to skip critic update for the $i$-th iteration when its critic loss $\gL_{\mathrm{critic}, i}$ exceed the critic loss threshold $\ell_{\mathrm{critic}}$.

We remark that the mentioned reward-clapping, augmented termination condition, Huber loss and critic-update skipping are mainly engineering tricks for improving training stablity in compensation for not knowing the true reward function and the true termination condition.

\textbf{Discriminator Training.}
To form the ``fake'' sample $\gB_{\mathrm{fake}}$ for training the discriminator, we first sample $\abs{\gB}$ states $s \sim \gD$, and remove the terminal states.
Second, we apply the state-smoothing trick in \citet{jointmatching2022} to $s$ with noise standard deviation $\sigma_J$ using $\epsilon \sim \gN\br{\vzero, \sigma_J^2\mI}, \, s \leftarrow s + \epsilon$. 
Third, get a corresponding action using $a\sim \pi_\vphi\br{\cdot \given s}$ for each $s$.
Fourth, get a sample of the next state $s'$ using the learned dynamics $\widehat P$ as $s' \sim \widehat P(\cdot \given s, a)$, with terminal states removed.
Fifth, sample one next action $a'$ for each next state $s'$ via $a' \sim \pi_\vphi\br{\cdot \given s'}$. 
Finally, the ``fake" sample $\gB_{\mathrm{fake}}$ is defined as 
\begin{equation*}
\gB_{\mathrm{fake}} \triangleq     
\begin{bmatrix}
    (s, & a) \\
    (s', & a')
    \end{bmatrix}.
\end{equation*}

The ``true" sample $\gB_{\mathrm{true}}$ is formed by sampling $\abs{\gB_{\mathrm{fake}}}$ state-action tuples from $\denv$.

Optimize the discriminator $D_\vw$ to maximize
\begin{equation*}
    \frac{1}{\abs{\gB_{\mathrm{true}}}} \sum_{(s,a)\sim \gB_{\mathrm{true}}}\sbr{\log D_\vw(s,a)} + \frac{1}{\abs{\gB_{\mathrm{fake}}}} \sum_{(s,a)\sim \gB_{\mathrm{fake}}} \sbr{\log\br{1-D_\vw(s,a)}},
\end{equation*}
with respect to $\vw$ with learning rate $\eta_\vw$.

We remark that the state-smoothing trick is applied here to smooth the discrete state-distribution manifested in $\gD$ and to provide a better coverage of the state space, which is similar to a kernel density approximation \citep{nonparamstat2006} of $d\br{s,a} = f \cdot d_{\pi_b}^{P^*}(s,a) + (1-f) \cdot d_{\pi_\vphi}^{\widehat P} (s,a)$ using data points $s \in \gD$ and with radial basis kernel of bandwidth $\sigma_J$.
The terminal states are removed since no action choice occurs on them by definition.

\textbf{Actor Training.}
We update the policy every $k$ iterations.
In training the actor, we first calculate the generator loss $$\gL_g(\vphi) = \frac{1}{\abs{\gB_{\mathrm{fake}}}} \sum_{(s,a) \in \gB_{\mathrm{fake}}}\sbr{\log\br{1-D_\vw\br{s,a}}}$$ using the ``fake" sample $\gB_{\mathrm{fake}}$ and discriminator $D_\vw$.

Then we optimize policy with learning rate $\eta_\vphi$ for the target
\begin{equation*}
    \arg\min_{\vphi} -\lambda \cdot\frac{1}{\abs{\gB}} \sum_{s \in \gB, a \sim \pi_\vphi(\cdot \given s)}\sbr{\min_{j=1,2} Q_{\vtheta_j}(s, a)} + \gL_g(\vphi),
\end{equation*}
where the regularization coefficient $\lambda$ is defined similar to \citet{td3bc2021} as $\lambda = \alpha/Q_{avg}$, with soft-updated $Q_{avg}$ and $\alpha=10$ across all datasets.

\textbf{Soft updates.}
Similar to \citet{td32018}, we perform soft-update on the target-network parameter, the $Q_{avg}$ and the critic loss threshold $\ell_{\mathrm{critic}}$.
The first three is soft-updated after each iteration while the last is updated after each epoch.
\begin{equation*}
    \begin{split}
        \vphi' & \leftarrow \beta \vphi + (1-\beta) \vphi', \\
        \vtheta'_j &\leftarrow \beta \vtheta_j + (1-\beta)\vtheta'_j, \, \forall\, j=1,2, \\
        Q_{avg} &\leftarrow \beta \cdot \frac{1}{\abs{\gB}} \sum_{\br{s, a} \in \gB} \abs{Q\br{s, a}} + (1 - \beta) \cdot Q_{avg}, \\
        \ell_{\mathrm{critic}} &\leftarrow 0.05 \cdot \br{\mu_{\gL_{\mathrm{critic}, i}} + 3\sigma_{\gL_{\mathrm{critic}, i}}} + (1-0.05) \cdot \ell_{\mathrm{critic}},\; \text{for }\gL_{\mathrm{critic}, i} \text{ in current epoch},
    \end{split}
\end{equation*}
where $0.05$ and $3$ in the last equation are some conveniently chosen numbers.

\textbf{Warm-start step.} 
We follow \citet{cql2020} and \citet{idac2020} to devote the first $40$ epochs of training for warm start.
In this phase, we sample $\gB \sim \denv$.
We do not train the critic during the warm-start phase and optimize policy with respect to the generator loss $\gL_g(\vphi)$ only, with the same learning rate $\eta_\vphi$, \textit{i.e.}, $\arg\min_\vphi \gL_g(\vphi)$.

\subsection{Details for the Reinforcement Learning Experiment} \label{sec:rl_exp_details}

{\bfseries Datasets.}
We use the continuous control tasks provided by the D4RL dataset \citep{fu2021d4rl} to conduct algorithmic evaluations.
Due to limited computational resources, we follow the literature to select the ``medium-expert," ``medium-replay," and ``medium" datasets for the Hopper, HalfCheetah, Walker2d tasks in the Gym-MuJoCo domain, which are commonly used % as they are popular
benchmarks in prior work \citep{bcq2019, bear2019, brac2019, cql2020}.
We follow the literature \citep{mabe2021,decisiontrans2021,iql2021} to not test on the ``random" and ``expert" datasets as these datasets are less practical \citep{bremen2021} and can be respectively solved by directly using standard off-policy RL algorithms \citep{rem2020} and the behavior cloning algorithms.
We note that a comprehensive benchmarking of prior offline-RL algorithms on the ``expert” datasets is currently unavailable in the literature, which is out of the scope of this paper.
Apart from the Gym-MuJoCo, we also consider the Maze2D tasks\footnote{We use the tasks ``maze2d-umaze," ``maze2d-medium," and ``maze2d-large.''} for their non-Markovian data-collecting policy and the Adroit tasks\footnote{We use the tasks ``pen-human,'' ``pen-cloned," ``pen-expert," and ``door-expert."} \citep{adroit2018} for their sparse reward-signal and high dimensionality.

{
Due to limited computational resources, we choose not to evaluate methods on the full set of Adriot datasets.
From the benchmarking results in the D4RL whitepaper, most of other Adroit datasets are out of the scope of current offline RL algorithms.
As an example, we evaluate our method and some baselines on the other two ``door'' datasets and present the results in Table~\ref{table:results_door_2}.
We hypothesize that obtaining good performance on these datasets requires specially-designed methods and a deeper investigation onto the quality of those datasets.
Similar hypothesis may also be applied to other Adroit datasets, and thus we do not test on those datasets.
}
\vspace{-1.6em}
\begin{table}[H]
% \scriptsize
% \captionsetup{font=scriptsize}
\caption{Average scores of some baselines and our SDM-GAN on the other two ``door" datasets.}
% \vspace{-1.3em}
\label{table:results_door_2}
\centering 
\setlength{\tabcolsep}{4.5pt}
\def\arraystretch{1.}
\begin{tabular}{@{}lccccc@{}}
\toprule
Task Name  & FisherBRC & TD3+BC & WMOPO & SDM-GAN \\ \midrule
door-cloned &  0    $\pm$ {\footnotesize 0.1}     & -0.3 $\pm$ {\footnotesize 0}  & -0.1   $\pm$ {\footnotesize 0.1}   & 0   $\pm$ {\footnotesize 0.1}    \\
door-human  &  0.1   $\pm$ {\footnotesize 0}    & -0.3 $\pm$ {\footnotesize 0}  & -0.2   $\pm$ {\footnotesize 0.2}   & 0.2  $\pm$ {\footnotesize 0.6}  \\ \bottomrule
\end{tabular}
\end{table}
% \end{wraptable}
\vspace{-1.6em}

{\bfseries Evaluation Protocol.}
In all the experiments, we follow \citet{fu2021d4rl} to use 3 random seeds for evaluation and to use the ``v0" version of the datasets in the Gym-MuJoCo and Adroit domains.
In our preliminary study, we find that the rollout results of some existing algorithms can be unstable across epochs, even towards the end of training.
To reduce the instability during the evaluation, for our algorithm, we report the mean and standard deviation of the last five rollouts, conducted at the end of the last five training epochs, across three random seeds  $\cbr{0,1,2}$.
We train our algorithm for $1000$ epochs, where each epoch consists of $1000$ mini-batch stochastic gradient descent steps.
For evaluation, We rollout our agent and the baselines for $10$ episodes after each epoch of training.

{
{\bfseries Source Code.}
The design of our \href{https://github.com/Shentao-YANG/SDM-GAN_ICML2022}{source code} for the RL experiments is motivated by the codebase of \citet{lu2021lisp}.
}

\subsubsection{Details for the Implementation via GAN}  \label{sec:algo_details_gan}

In practice, the rollouts contained in the offline dataset have finite horizon, and thus in calculating the Bellman update target, special treatment is needed per appearance of the terminal states.
We follow the standard treatment \citep{dqn2013,rlintro2018} to modify the critic update target Equation~\eqref{eq:critic_target} as 
\begin{equation*}
    \widetilde{Q}\br{\vs,\va} = \begin{cases}
    r(s,a).\mathrm{clamp}(r_{\mathrm{min}}, r_{\mathrm{max}}) + \gamma \cdot \widetilde{Q}\br{s, a} \cdot \br{\abs{\widetilde{Q}\br{s, a}} < 2000} & \text{ if $s$ is a non-terminal state} \\
    r(s,a).\mathrm{clamp}(r_{\mathrm{min}}, r_{\mathrm{max}}) & \text{ if $s$ is a terminal state}
    \end{cases}.
\end{equation*}
This modification removes the contribution of the next state when reaching the terminal state.

Following \citet{samplegenerative2016}, we choose the noise distribution $p_z\br{z}$ as the multivariate standard normal distribution, where the dimension of $z$ is defaulted as $\mathrm{dim}\br{z} = {\tt min(10, \, state\_dim // 2)}$.
To sample from the implicit policy, for each state $s$, we first independently sample $z \sim \gN\br{\vzero, \mI}$.
We then concatenate $s$ with $z$ and feed the resulting $\sbr{s, z}$ into the deterministic policy network to generate stochastic actions.
To sample from a small region around the next state $s'$ (Appendix \ref{sec:algo_details_ac}), we keep the original $s'$ and repeat it additionally $N_B-1$ times.
For each of the $N_B-1$ replications, we add an independent Gaussian noise $\epsilon \sim \gN(\vzero, \sigma_B^2 \mI)$.
The original $s'$ and its $N_B-1$ noisy replications are then fed into the implicit policy to sample the corresponding actions.

For fair comparison, we use the same network architecture as the implementation of the BCQ \citep{bcq2019}, which will be presented in details in Appendix \ref{sec:algo_details_gan}. 
Due to limited computational resources, we leave a fine-tuning of the noise distribution $p_z\br{z}$, the network architectures, and the optimization hyperparameters for future work.

To approximately match the Jensen--Shannon divergence between the current and the behavior policies via GAN, a crucial step is to stably and effectively train the GAN structure.
Together, we adopt the following technique from the literature.
\begin{itemize}[leftmargin=*]
    \item Following \citet{gan2014}, we train $\pi_\vphi$ to maximize
    $
        \E_{(s,a) \in \gB_{\mathrm{fake}}} \sbr{\log\br{D_\vw(s,a)}}.
    $
    
    % To provide stronger gradients  in the early training, rather than training the policy $\pi_\vphi$ to minimize
    % \begin{equation*} \textstyle
    %     \E_{(s,a) \in \gB_{\mathrm{fake}}} \sbr{\log\br{1-D_\vw(s,a)}}
    % \end{equation*}
    % we follow \citep{gan2014} to train $\pi_\vphi$ to maximize
    % \begin{equation*} \textstyle
    %     \E_{(s,a) \in \gB_{\mathrm{fake}}} \sbr{\log\br{D_\vw(s,a)}}
    % \end{equation*}
    \item Motivated by \citet{dcgan2016}, we use the LeakyReLU activation in both the generator and discriminator, under the default {\tt negative\_slope=0.01}.
    \item To stabilize the training, we follow \citet{dcgan2016} to use a reduced momentum term $\beta_1 = 0.4$ in the Adam optimizer \citep{adam2014} of the policy network and the discriminator network, with learning rate $2 \times 10^{-4}$.
    % $\eta_\vphi = \eta_\vw = $2 \times 10^{-4}$.
    % \item We follow \citep{dcgan2016} to use actor and discriminator learning rate $\eta_\vphi = \eta_\vw = 2 \times 10^{-4}$.
    \item To avoid overfitting of the discriminator, we are motivated by \citet{improvetechgan2016} and \citet{gantutorial2017} to use one-sided label smoothing with soft and noisy labels.
    Specifically, the labels for the ``true" sample $\gB_{\mathrm{true}}$ is replaced with a random number between $0.8$ and $1.0$.
    No label smoothing is applied for the ``fake" sample $\gB_{\mathrm{fake}}$, and therefore their labels are all $0$.
    \item The loss function for training the discriminator in GAN is the binary cross entropy between the labels and the outputs from the discriminator.
    \item Motivated by TD3 \citep{td32018} and GAN, we update $\pi_\vphi(\cdot \given s)$ once per $k$ updates of the critics and discriminator. 
\end{itemize}

Table \ref{table:gan_param} shows the hyperparameters shared across all datasets for our empirical study.
Note that several simplifications are made to minimize hyperparameter tuning, such as fixing $\eta_\vphi = \eta_\vw$ as in \citet{dcgan2016} and $\sigma_B = \sigma_J$.
We also fix $N_a = 1$ to ease computation. We comment that many of these hyperparameters can be set based on literature, for example, we use $\eta_\vphi = \eta_\vw = 2 \times 10^{-4}$ as in \citet{dcgan2016}, $\eta_\vtheta = 3 \times 10^{-4}$ and $N_{\mathrm{warm}} = 40$ as in \citet{cql2020}, $c=0.75$ as in \citet{bcq2019}, policy frequency $k=2$ as in \citet{td32018}, $f=0.5$ as in \citet{combo2021} and model learning rate $\eta_{\widehat P} = 0.001$ as in \citet{mopo2020}. Unless specified otherwise, the same hyperparameters, shared and non-shared, are used for both the main results and the ablation study.

\begin{table}[ht]
\caption{Shared hyperparameters for the GAN implementation.} \label{table:gan_param}
\centering
\begin{tabular}{@{}l|l@{}}
\toprule
Hyperparameter & Value \\ \midrule
Optimizer      & Adam \cite{adam2014}  \\
Learning rate $\eta_\vtheta$ &  $3 \times 10^{-4}$  \\ 
Learning rate $\eta_\vphi$, $\eta_\vw$ & $2 \times 10^{-4}$      \\ 
Learning rate $\eta_{\widehat P}$ &  $1 \times 10^{-3}$  \\ 
Penalty coefficient $\alpha$ & $10.0$ \\
Evaluation frequency (epoch length) & $10^3$ \\
Training iterations & $10^6$ \\
Batch size & $512$ \\
Discount factor $\gamma$ & $0.99$ \\
Target network update rate $\beta$ & $0.005$ \\
Weighting for clipped double Q-learning $c$ & $0.75$ \\
Noise distribution $p_z(z)$ & $\gN\br{\vzero, \mI}$ \\
Standard deviations for state smoothing $\sigma_B, \sigma_J$ & $3 \times 10^{-4}$ \\
Number of smoothed states in Bellman backup $N_B$ & $50$ \\
Number of actions $\hat\va$ at each $\hat \vs$ $N_a$ & $1$ \\
Policy frequency $k$ & $2$ \\
Rollout generation frequency & per $250$ iterations \\
Number of model-rollout samples per iteration & $128$ \\
Rollout retain epochs & $5$ \\
Real data percentage $f$ & $0.5$  \\
Random seeds & $\cbr{0,1,2}$ \\
\bottomrule
\end{tabular}
\end{table}

Due to the diverse nature of the tested datasets, we follow the common practice in model-based offline RL to perform gentle hyperparameter tuning for each task. 
Specifically, we tune the dimension of the noise distribution $p_z(z)$ for controlling the stochasticity of the learned policy, and the rollout horizon $h$ for mitigating model estimation error.

As shown in \citet{act2021} and \citet{zhang2021alignment}, a larger noise dimension, such as $50$, can facilitate learning a more flexible distribution. 
Hence we use the default noise dimension for the tasks:
halfcheetah-medium-expert, hopper-medium-expert, halfcheetah-medium,  walker2d-medium, maze2d-umaze, pen-cloned, door-expert, pen-human; 
and noise dimension $50$ for the tasks:walker2d-medium-expert, hopper-medium, halfcheetah-medium-replay, hopper-medium-replay, walker2d-medium-replay, maze2d-medium, maze2d-large, pen-expert.

Similar to \citet{mbpo2019} and \citet{mopo2020} we consider the rollout horizon $h\in \cbr{1,3,5}$. 
We use $h=1$ for hopper-medium, walker2d-medium, hopper-medium-replay, pen-cloned;
$h=3$ for 
halfcheetah-medium-expert, halfcheetah-medium, halfcheetah-medium-replay, maze2d-umaze, maze2d-medium, maze2d-large, pen-expert, door-expert, pen-human;
and $h=5$ for
hopper-medium-expert, walker2d-medium-expert, walker2d-medium-replay.

Below, we state the network architectures of the actor, critic, and the discriminator in the implementation via GAN.
Note that we use two critic networks with the same architecture to perform clipped double Q-learning.

\begin{multicols}{2}% 2-column layout
  \begin{minipage}{0.45\textwidth}
    Actor
\begin{verbatim}
Linear(state_dim+noise_dim, 400)
LeakyReLU
Linear(400, 300)
LeakyReLU
Linear(300, action_dim)
max_action * tanh
\end{verbatim}
  \end{minipage}
  
\begin{minipage}{0.45\textwidth}
Critic 
\begin{verbatim}
Linear(state_dim+action_dim, 400)
LeakyReLU
Linear(400, 300)
LeakyReLU
Linear(300, 1)
\end{verbatim}
\end{minipage}
\end{multicols}

\begin{multicols}{2}
\begin{minipage}{0.45\textwidth}
Discriminator in GAN
\begin{verbatim}
Linear(state_dim+action_dim, 400)
LeakyReLU
Linear(400, 300)
LeakyReLU
Linear(300, 1)
Sigmoid
\end{verbatim}
\end{minipage}

\begin{minipage}{0.45\textwidth}
Discriminator in WGAN (Ablation Study in Section~\ref{sec:exp_ablation})
\begin{verbatim}
Linear(state_dim+action_dim, 400)
LeakyReLU
Linear(400, 300)
LeakyReLU
Linear(300, 1)
\end{verbatim}
\end{minipage}
\end{multicols}

\subsubsection{Results of CQL} \label{sec:tech_cql}

We note that the official CQL GitHub repository does not provide hyperparameter settings for the Maze2D and Adroit domain of tasks.
For datasets in these two domains, we train CQL agents using five hyperparameter settings: the four recommended Gym-MuJoCo settings and the one recommended Ant-Maze setting.
We then calculate the average normalized-return over the random seeds $\cbr{0,1,2}$ for each hyperparameter settings and per-dataset select the best result from these five settings.
We comment that this per-dataset tuning may give CQL some advantage on the Maze2D and Adroit domains, and is a compensation for the missing of recommended hyperparameters.
For the Gym-MuJoCo domain, we follow the recommendation by \citet{cql2020}.

%% file: main.bbl
\begin{thebibliography}{87}
\providecommand{\natexlab}[1]{#1}
\providecommand{\url}[1]{\texttt{#1}}
\expandafter\ifx\csname urlstyle\endcsname\relax
  \providecommand{\doi}[1]{doi: #1}\else
  \providecommand{\doi}{doi: \begingroup \urlstyle{rm}\Url}\fi

\bibitem[Agarwal et~al.(2020)Agarwal, Schuurmans, and Norouzi]{rem2020}
Agarwal, R., Schuurmans, D., and Norouzi, M.
\newblock An optimistic perspective on offline reinforcement learning.
\newblock In \emph{International Conference on Machine Learning}, pp.\
  104--114. PMLR, 2020.

\bibitem[Bellemare et~al.(2017)Bellemare, Dabney, and Munos]{c512017}
Bellemare, M.~G., Dabney, W., and Munos, R.
\newblock {A Distributional Perspective on Reinforcement Learning}.
\newblock In \emph{{International Conference on Machine Learning}}, 2017.

\bibitem[Binkowski et~al.(2018)Binkowski, Sutherland, Arbel, and
  Gretton]{mmdgangp2018}
Binkowski, M., Sutherland, D.~J., Arbel, M., and Gretton, A.
\newblock {Demystifying MMD GANs}.
\newblock \emph{ArXiv}, abs/1801.01401, 2018.

\bibitem[Cang et~al.(2021)Cang, Rajeswaran, Abbeel, and Laskin]{mabe2021}
Cang, C., Rajeswaran, A., Abbeel, P., and Laskin, M.
\newblock {Behavioral Priors and Dynamics Models: Improving Performance and
  Domain Transfer in Offline RL}.
\newblock \emph{ArXiv}, abs/2106.09119, 2021.

\bibitem[Casella \& Berger(2001)Casella and Berger]{casellaberger2001}
Casella, G. and Berger, R.
\newblock \emph{Statistical Inference}.
\newblock {Duxbury Resource Center}, June 2001.

\bibitem[Chen et~al.(2021)Chen, Lu, Rajeswaran, Lee, Grover, Laskin, Abbeel,
  Srinivas, and Mordatch]{decisiontrans2021}
Chen, L., Lu, K., Rajeswaran, A., Lee, K., Grover, A., Laskin, M., Abbeel, P.,
  Srinivas, A., and Mordatch, I.
\newblock {Decision Transformer: Reinforcement Learning via Sequence Modeling}.
\newblock \emph{ArXiv}, abs/2106.01345, 2021.

\bibitem[Chua et~al.(2018)Chua, Calandra, McAllister, and Levine]{pets2018}
Chua, K., Calandra, R., McAllister, R., and Levine, S.
\newblock {Deep Reinforcement Learning in a Handful of Trials using
  Probabilistic Dynamics Models}.
\newblock In \emph{{Advances in neural information processing systems}}, 2018.

\bibitem[Clavera et~al.(2018)Clavera, Rothfuss, Schulman, Fujita, Asfour, and
  Abbeel]{mbmpo2018}
Clavera, I., Rothfuss, J., Schulman, J., Fujita, Y., Asfour, T., and Abbeel, P.
\newblock {Model-Based Reinforcement Learning via Meta-Policy Optimization}.
\newblock \emph{{ArXiv}}, abs/1809.05214, 2018.

\bibitem[Dabney et~al.(2018)Dabney, Rowland, Bellemare, and Munos]{qrdqn2017}
Dabney, W., Rowland, M., Bellemare, M., and Munos, R.
\newblock Distributional reinforcement learning with quantile regression.
\newblock In \emph{Proceedings of the AAAI Conference on Artificial
  Intelligence}, volume~32, 2018.

\bibitem[Ernst et~al.(2005)Ernst, Geurts, and Wehenkel]{treerl2005}
Ernst, D., Geurts, P., and Wehenkel, L.
\newblock {Tree-Based Batch Mode Reinforcement Learning.}
\newblock \emph{{J. Mach. Learn. Res.}}, 6:\penalty0 503--556, 2005.

\bibitem[Fan et~al.(2020)Fan, Zhang, Chen, and Zhou]{fan2020bayesian}
Fan, X., Zhang, S., Chen, B., and Zhou, M.
\newblock Bayesian attention modules.
\newblock \emph{Advances in Neural Information Processing Systems},
  33:\penalty0 16362--16376, 2020.

\bibitem[Fan et~al.(2021)Fan, Zhang, Tanwisuth, Qian, and
  Zhou]{fan2021contextual}
Fan, X., Zhang, S., Tanwisuth, K., Qian, X., and Zhou, M.
\newblock Contextual dropout: An efficient sample-dependent dropout module.
\newblock \emph{arXiv preprint arXiv:2103.04181}, 2021.

\bibitem[Ferguson(1996)]{ferguson1996}
Ferguson, T.~S.
\newblock \emph{{A Course in Large Sample Theory}}.
\newblock {Chapman \& Hall}, London, 1996.

\bibitem[Fu et~al.(2019)Fu, Kumar, Soh, and Levine]{diagbottleneck2019}
Fu, J., Kumar, A., Soh, M., and Levine, S.
\newblock {Diagnosing Bottlenecks in Deep Q-learning Algorithms}.
\newblock In \emph{{International Conference on Machine Learning}}, 2019.

\bibitem[Fu et~al.(2020)Fu, Kumar, Nachum, Tucker, and Levine]{fu2021d4rl}
Fu, J., Kumar, A., Nachum, O., Tucker, G., and Levine, S.
\newblock D4rl: Datasets for deep data-driven reinforcement learning.
\newblock \emph{arXiv preprint arXiv:2004.07219}, 2020.

\bibitem[Fujimoto \& Gu(2021)Fujimoto and Gu]{td3bc2021}
Fujimoto, S. and Gu, S.~S.
\newblock {A Minimalist Approach to Offline Reinforcement Learning}.
\newblock \emph{{ArXiv}}, abs/2106.06860, 2021.

\bibitem[Fujimoto et~al.(2018)Fujimoto, van Hoof, and Meger]{td32018}
Fujimoto, S., van Hoof, H., and Meger, D.
\newblock {Addressing Function Approximation Error in Actor-Critic Methods}.
\newblock In Dy, J. and Krause, A. (eds.), \emph{{Proceedings of the 35th
  International Conference on Machine Learning}}, volume~80 of
  \emph{Proceedings of Machine Learning Research}, pp.\  1587--1596. PMLR,
  10--15 Jul 2018.

\bibitem[Fujimoto et~al.(2019)Fujimoto, Meger, and Precup]{bcq2019}
Fujimoto, S., Meger, D., and Precup, D.
\newblock {Off-Policy Deep Reinforcement Learning without Exploration}.
\newblock In Chaudhuri, K. and Salakhutdinov, R. (eds.), \emph{{Proceedings of
  the 36th International Conference on Machine Learning}}, volume~97 of
  \emph{Proceedings of Machine Learning Research}, pp.\  2052--2062. PMLR,
  09--15 Jun 2019.

\bibitem[Gilotte et~al.(2018)Gilotte, Calauz{\`e}nes, Nedelec, Abraham, and
  Doll{\'e}]{offlineabtest2018}
Gilotte, A., Calauz{\`e}nes, C., Nedelec, T., Abraham, A., and Doll{\'e}, S.
\newblock {Offline A/B Testing for Recommender Systems}.
\newblock \emph{Proceedings of the Eleventh ACM International Conference on Web
  Search and Data Mining}, 2018.

\bibitem[Goodfellow(2016)]{gantutorial2017}
Goodfellow, I.
\newblock Nips 2016 tutorial: Generative adversarial networks.
\newblock \emph{arXiv preprint arXiv:1701.00160}, 2016.

\bibitem[Goodfellow et~al.(2014)Goodfellow, Pouget-Abadie, Mirza, Xu,
  Warde-Farley, Ozair, Courville, and Bengio]{gan2014}
Goodfellow, I., Pouget-Abadie, J., Mirza, M., Xu, B., Warde-Farley, D., Ozair,
  S., Courville, A., and Bengio, Y.
\newblock {Generative Adversarial Nets}.
\newblock In Ghahramani, Z., Welling, M., Cortes, C., Lawrence, N., and
  Weinberger, K.~Q. (eds.), \emph{{Advances in Neural Information Processing
  Systems}}, volume~27. Curran Associates, Inc., 2014.

\bibitem[Gretton et~al.(2012)Gretton, Borgwardt, Rasch, Sch{\"o}lkopf, and
  Smola]{kerneltwosampletest2012}
Gretton, A., Borgwardt, K., Rasch, M., Sch{\"o}lkopf, B., and Smola, A.
\newblock {A Kernel Two-Sample Test}.
\newblock \emph{J. Mach. Learn. Res.}, 13:\penalty0 723--773, 2012.

\bibitem[Gulcehre et~al.(2021)Gulcehre, Colmenarejo, ziyu wang, Sygnowski,
  Paine, Zolna, Chen, Hoffman, Pascanu, and de~Freitas]{addressingextra2021}
Gulcehre, C., Colmenarejo, S.~G., ziyu wang, Sygnowski, J., Paine, T., Zolna,
  K., Chen, Y., Hoffman, M., Pascanu, R., and de~Freitas, N.
\newblock {Addressing Extrapolation Error in Deep Offline Reinforcement
  Learning}.
\newblock 2021.

\bibitem[Gulrajani et~al.(2017)Gulrajani, Ahmed, Arjovsky, Dumoulin, and
  Courville]{wgangp2017}
Gulrajani, I., Ahmed, F., Arjovsky, M., Dumoulin, V., and Courville, A.~C.
\newblock {Improved Training of Wasserstein GANs}.
\newblock In \emph{{Advances in neural information processing systems}}, 2017.

\bibitem[Haarnoja et~al.(2018)Haarnoja, Zhou, Hartikainen, Tucker, Ha, Tan,
  Kumar, Zhu, Gupta, Abbeel, and Levine]{sacnew2018}
Haarnoja, T., Zhou, A., Hartikainen, K., Tucker, G., Ha, S., Tan, J., Kumar,
  V., Zhu, H., Gupta, A., Abbeel, P., and Levine, S.
\newblock {Soft Actor-Critic Algorithms and Applications}.
\newblock \emph{ArXiv}, abs/1812.05905, 2018.

\bibitem[Hishinuma \& Senda(2021)Hishinuma and Senda]{wmopo2021}
Hishinuma, T. and Senda, K.
\newblock {Weighted model estimation for offline model-based reinforcement
  learning}.
\newblock In \emph{{Advances in neural information processing systems}}, 2021.

\bibitem[Janner et~al.(2019)Janner, Fu, Zhang, and Levine]{mbpo2019}
Janner, M., Fu, J., Zhang, M., and Levine, S.
\newblock {When to Trust Your Model: Model-Based Policy Optimization}.
\newblock \emph{{ArXiv}}, abs/1906.08253, 2019.

\bibitem[Janner et~al.(2021)Janner, Li, and Levine]{trajectorytransformer2021}
Janner, M., Li, Q., and Levine, S.
\newblock {Reinforcement Learning as One Big Sequence Modeling Problem}.
\newblock \emph{{ArXiv}}, abs/2106.02039, 2021.

\bibitem[Jaques et~al.(2019)Jaques, Ghandeharioun, Shen, Ferguson, Lapedriza,
  Jones, Gu, and Picard]{offlinerldialog2019}
Jaques, N., Ghandeharioun, A., Shen, J.~H., Ferguson, C., Lapedriza, {\`A}.,
  Jones, N.~J., Gu, S., and Picard, R.~W.
\newblock {Way Off-Policy Batch Deep Reinforcement Learning of Implicit Human
  Preferences in Dialog}.
\newblock \emph{ArXiv}, abs/1907.00456, 2019.

\bibitem[Jiang \& Li(2016)Jiang and Li]{jiang16doubly}
Jiang, N. and Li, L.
\newblock Doubly robust off-policy evaluation for reinforcement learning.
\newblock In \emph{Proceedings of the 23rd International Conference on Machine
  Learning}, pp.\  652--661, 2016.

\bibitem[Kallus \& Zhou(2020)Kallus and Zhou]{confoundingrobust2020}
Kallus, N. and Zhou, A.
\newblock {Confounding-robust policy evaluation in infinite-horizon
  reinforcement learning}.
\newblock \emph{{Advances in Neural Information Processing Systems}},
  33:\penalty0 22293--22304, 2020.

\bibitem[Kidambi et~al.(2020)Kidambi, Rajeswaran, Netrapalli, and
  Joachims]{morel2020}
Kidambi, R., Rajeswaran, A., Netrapalli, P., and Joachims, T.
\newblock {MOReL : Model-Based Offline Reinforcement Learning}.
\newblock \emph{ArXiv}, abs/2005.05951, 2020.

\bibitem[Kingma \& Ba(2014)Kingma and Ba]{adam2014}
Kingma, D.~P. and Ba, J.
\newblock {Adam: A Method for Stochastic Optimization}.
\newblock In \emph{{International Conference on Learning Representations}},
  2014.

\bibitem[Kostrikov et~al.(2021{\natexlab{a}})Kostrikov, Nair, and
  Levine]{iql2021}
Kostrikov, I., Nair, A., and Levine, S.
\newblock {Offline Reinforcement Learning with Implicit Q-Learning}.
\newblock \emph{{ArXiv}}, abs/2110.06169, 2021{\natexlab{a}}.

\bibitem[Kostrikov et~al.(2021{\natexlab{b}})Kostrikov, Tompson, Fergus, and
  Nachum]{fisherbrc2021}
Kostrikov, I., Tompson, J., Fergus, R., and Nachum, O.
\newblock {Offline Reinforcement Learning with Fisher Divergence Critic
  Regularization}.
\newblock In \emph{{International Conference on Machine Learning}},
  2021{\natexlab{b}}.

\bibitem[Kumar et~al.(2019)Kumar, Fu, Soh, Tucker, and Levine]{bear2019}
Kumar, A., Fu, J., Soh, M., Tucker, G., and Levine, S.
\newblock {Stabilizing Off-Policy Q-Learning via Bootstrapping Error
  Reduction}.
\newblock In \emph{{Advances in Neural Information Processing Systems}},
  volume~32. Curran Associates, Inc., 2019.

\bibitem[Kumar et~al.(2020)Kumar, Zhou, Tucker, and Levine]{cql2020}
Kumar, A., Zhou, A., Tucker, G., and Levine, S.
\newblock {Conservative Q-Learning for Offline Reinforcement Learning}.
\newblock In Larochelle, H., Ranzato, M., Hadsell, R., Balcan, M.~F., and Lin,
  H. (eds.), \emph{{Advances in Neural Information Processing Systems}},
  volume~33, pp.\  1179--1191. Curran Associates, Inc., 2020.

\bibitem[Kurutach et~al.(2018)Kurutach, Clavera, Duan, Tamar, and
  Abbeel]{metrpo2018}
Kurutach, T., Clavera, I., Duan, Y., Tamar, A., and Abbeel, P.
\newblock {Model-Ensemble Trust-Region Policy Optimization}.
\newblock \emph{{International Conference on Learning Representations}},
  abs/1802.10592, 2018.

\bibitem[Lange et~al.(2012)Lange, Gabel, and Riedmiller]{batchrl2012}
Lange, S., Gabel, T., and Riedmiller, M.
\newblock \emph{{Batch Reinforcement Learning}}, pp.\  45--73.
\newblock Springer Berlin Heidelberg, Berlin, Heidelberg, 2012.
\newblock ISBN 978-3-642-27645-3.
\newblock \doi{10.1007/978-3-642-27645-3_2}.

\bibitem[Laroche \& Trichelair(2019)Laroche and Trichelair]{sibb2019}
Laroche, R. and Trichelair, P.
\newblock {Safe Policy Improvement with Baseline Bootstrapping}.
\newblock In \emph{{International Conference on Machine Learning}}, 2019.

\bibitem[Lee et~al.(2021{\natexlab{a}})Lee, Lee, and Kim]{repbsde2021}
Lee, B.-J., Lee, J., and Kim, K.-E.
\newblock {Representation Balancing Offline Model-based Reinforcement
  Learning}.
\newblock In \emph{{International Conference on Learning Representations}},
  2021{\natexlab{a}}.

\bibitem[Lee et~al.(2021{\natexlab{b}})Lee, Jeon, Lee, Pineau, and
  Kim]{optidice2021}
Lee, J., Jeon, W., Lee, B.-J., Pineau, J., and Kim, K.-E.
\newblock {OptiDICE: Offline Policy Optimization via Stationary Distribution
  Correction Estimation}.
\newblock \emph{ArXiv}, abs/2106.10783, 2021{\natexlab{b}}.

\bibitem[Levine et~al.(2020)Levine, Kumar, Tucker, and Fu]{offlinetutorial2020}
Levine, S., Kumar, A., Tucker, G., and Fu, J.
\newblock Offline reinforcement learning: Tutorial, review, and perspectives on
  open problems.
\newblock \emph{arXiv preprint arXiv:2005.01643}, 2020.

\bibitem[Liao et~al.(2020)Liao, Qi, Klasnja, and Murphy]{liao2020batch}
Liao, P., Qi, Z., Klasnja, P., and Murphy, S.
\newblock Batch policy learning in average reward markov decision processes.
\newblock \emph{arXiv preprint arXiv:2007.11771}, 2020.

\bibitem[Lillicrap et~al.(2016)Lillicrap, Hunt, Pritzel, Heess, Erez, Tassa,
  Silver, and Wierstra]{ddpg2016}
Lillicrap, T., Hunt, J.~J., Pritzel, A., Heess, N., Erez, T., Tassa, Y.,
  Silver, D., and Wierstra, D.
\newblock {Continuous Control with Deep Reinforcement Learning}.
\newblock \emph{CoRR}, abs/1509.02971, 2016.

\bibitem[Lin(1991)]{jsd1991}
Lin, J.
\newblock {Divergence Measures Based on the Shannon Entropy}.
\newblock \emph{IEEE Transactions on Information theory}, 37:\penalty0
  145--151, 1991.

\bibitem[Lin(1992)]{replaybuffer1992}
Lin, L.-J.
\newblock {Self-Improving Reactive Agents Based on Reinforcement Learning,
  Planning and Teaching}.
\newblock \emph{Machine Learning}, 8\penalty0 (3--4):\penalty0 293--321, 1992.

\bibitem[Liu et~al.(2018)Liu, Li, Tang, and Zhou]{breakingcurse2018}
Liu, Q., Li, L., Tang, Z., and Zhou, D.
\newblock {Breaking the Curse of Horizon: Infinite-Horizon Off-Policy
  Estimation}.
\newblock In \emph{{Advances in neural information processing systems}}, 2018.

\bibitem[Lu et~al.(2021)Lu, Grover, Abbeel, and Mordatch]{lu2021lisp}
Lu, K., Grover, A., Abbeel, P., and Mordatch, I.
\newblock Reset-free lifelong learning with skill-space planning.
\newblock In \emph{9th International Conference on Learning Representations,
  {ICLR} 2021, Virtual Event, Austria, May 3-7, 2021}, 2021.

\bibitem[Matsushima et~al.(2021)Matsushima, Furuta, Matsuo, Nachum, and
  Gu]{bremen2021}
Matsushima, T., Furuta, H., Matsuo, Y., Nachum, O., and Gu, S.~S.
\newblock {Deployment-Efficient Reinforcement Learning via Model-Based Offline
  Optimization}.
\newblock \emph{{International Conference on Learning Representations}},
  abs/2006.03647, 2021.

\bibitem[Mirza \& Osindero(2014)Mirza and Osindero]{cgan2014}
Mirza, M. and Osindero, S.
\newblock {Conditional Generative Adversarial Nets}.
\newblock \emph{ArXiv}, abs/1411.1784, 2014.

\bibitem[Mnih et~al.(2013)Mnih, Kavukcuoglu, Silver, Graves, Antonoglou,
  Wierstra, and Riedmiller]{dqn2013}
Mnih, V., Kavukcuoglu, K., Silver, D., Graves, A., Antonoglou, I., Wierstra,
  D., and Riedmiller, M.
\newblock Playing atari with deep reinforcement learning.
\newblock \emph{arXiv preprint arXiv:1312.5602}, 2013.

\bibitem[Mnih et~al.(2015)Mnih, Kavukcuoglu, Silver, Rusu, Veness, Bellemare,
  Graves, Riedmiller, Fidjeland, Ostrovski, Petersen, Beattie, Sadik,
  Antonoglou, King, Kumaran, Wierstra, Legg, and Hassabis]{dqn2015}
Mnih, V., Kavukcuoglu, K., Silver, D., Rusu, A.~A., Veness, J., Bellemare,
  M.~G., Graves, A., Riedmiller, M., Fidjeland, A.~K., Ostrovski, G., Petersen,
  S., Beattie, C., Sadik, A., Antonoglou, I., King, H., Kumaran, D., Wierstra,
  D., Legg, S., and Hassabis, D.
\newblock {Human-level control through deep reinforcement learning}.
\newblock \emph{Nature}, 518\penalty0 (7540):\penalty0 529--533, February 2015.
\newblock ISSN 00280836.

\bibitem[M{\"u}ller(1997)]{ipm1997}
M{\"u}ller, A.
\newblock {Integral Probability Metrics and Their Generating Classes of
  Functions}.
\newblock \emph{Advances in Applied Probability}, 29:\penalty0 429--443, 1997.

\bibitem[Nachum et~al.(2019{\natexlab{a}})Nachum, Chow, Dai, and
  Li]{dualdice2019}
Nachum, O., Chow, Y., Dai, B., and Li, L.
\newblock {DualDICE: Behavior-Agnostic Estimation of Discounted Stationary
  Distribution Corrections}.
\newblock In \emph{{Advances in neural information processing systems}},
  2019{\natexlab{a}}.

\bibitem[Nachum et~al.(2019{\natexlab{b}})Nachum, Dai, Kostrikov, Chow, Li, and
  Schuurmans]{algaedice2019}
Nachum, O., Dai, B., Kostrikov, I., Chow, Y., Li, L., and Schuurmans, D.
\newblock {AlgaeDICE: Policy Gradient from Arbitrary Experience}.
\newblock \emph{ArXiv}, abs/1912.02074, 2019{\natexlab{b}}.

\bibitem[Nie et~al.(2019)Nie, Brunskill, and Wager]{whentotreat2019}
Nie, X., Brunskill, E., and Wager, S.
\newblock {Learning When-to-Treat Policies}.
\newblock \emph{Journal of the American Statistical Association}, 116:\penalty0
  392 -- 409, 2019.

\bibitem[Puterman(2014)]{markovdp1994}
Puterman, M.~L.
\newblock \emph{Markov decision processes: discrete stochastic dynamic
  programming}.
\newblock John Wiley \& Sons, 2014.

\bibitem[Radford \& Sutskever(2018)Radford and Sutskever]{gpt2018}
Radford, A. and Sutskever, I.
\newblock {Improving Language Understanding by Generative Pre-Training}.
\newblock In \emph{{arxiv}}, 2018.

\bibitem[Radford et~al.(2016)Radford, Metz, and Chintala]{dcgan2016}
Radford, A., Metz, L., and Chintala, S.
\newblock {Unsupervised Representation Learning with Deep Convolutional
  Generative Adversarial Networks}.
\newblock \emph{CoRR}, abs/1511.06434, 2016.

\bibitem[Rajeswaran et~al.(2018)Rajeswaran, Kumar, Gupta, Schulman, Todorov,
  and Levine]{adroit2018}
Rajeswaran, A., Kumar, V., Gupta, A., Schulman, J., Todorov, E., and Levine, S.
\newblock {Learning Complex Dexterous Manipulation with Deep Reinforcement
  Learning and Demonstrations}.
\newblock \emph{ArXiv}, abs/1709.10087, 2018.

\bibitem[Rajeswaran et~al.(2020)Rajeswaran, Mordatch, and Kumar]{gamembrl2020}
Rajeswaran, A., Mordatch, I., and Kumar, V.
\newblock {A Game Theoretic Framework for Model Based Reinforcement Learning}.
\newblock In \emph{{International Conference on Machine Learning}}, 2020.

\bibitem[Ross \& Bagnell(2012)Ross and Bagnell]{asi2012}
Ross, S. and Bagnell, J.~A.
\newblock {Agnostic System Identification for Model-Based Reinforcement
  Learning}.
\newblock \emph{{International Conference on Machine Learning}}, abs/1203.1007,
  2012.

\bibitem[Salimans et~al.(2016)Salimans, Goodfellow, Zaremba, Cheung, Radford,
  and Chen]{improvetechgan2016}
Salimans, T., Goodfellow, I., Zaremba, W., Cheung, V., Radford, A., and Chen,
  X.
\newblock {Improved Techniques for Training GANs}.
\newblock In \emph{{Advances in neural information processing systems}}, 2016.

\bibitem[Siegel et~al.(2020)Siegel, Springenberg, Berkenkamp, Abdolmaleki,
  Neunert, Lampe, Hafner, Heess, and Riedmiller]{abm2020}
Siegel, N.~Y., Springenberg, J.~T., Berkenkamp, F., Abdolmaleki, A., Neunert,
  M., Lampe, T., Hafner, R., Heess, N., and Riedmiller, M.
\newblock Keep doing what worked: Behavioral modelling priors for offline
  reinforcement learning.
\newblock \emph{arXiv preprint arXiv:2002.08396}, 2020.

\bibitem[Silver et~al.(2016)Silver, Huang, Maddison, Guez, Sifre, van~den
  Driessche, Schrittwieser, Antonoglou, Panneershelvam, Lanctot, Dieleman,
  Grewe, Nham, Kalchbrenner, Sutskever, Lillicrap, Leach, Kavukcuoglu, Graepel,
  and Hassabis]{alphago2016}
Silver, D., Huang, A., Maddison, C.~J., Guez, A., Sifre, L., van~den Driessche,
  G., Schrittwieser, J., Antonoglou, I., Panneershelvam, V., Lanctot, M.,
  Dieleman, S., Grewe, D., Nham, J., Kalchbrenner, N., Sutskever, I.,
  Lillicrap, T., Leach, M., Kavukcuoglu, K., Graepel, T., and Hassabis, D.
\newblock {Mastering the Game of {Go} with Deep Neural Networks and Tree
  Search}.
\newblock \emph{Nature}, 529\penalty0 (7587):\penalty0 484--489, January 2016.
\newblock \doi{10.1038/nature16961}.

\bibitem[Silver et~al.(2017)Silver, Hubert, Schrittwieser, Antonoglou, Lai,
  Guez, Lanctot, Sifre, Kumaran, Graepel, Lillicrap, Simonyan, and
  Hassabis]{alphazero2017}
Silver, D., Hubert, T., Schrittwieser, J., Antonoglou, I., Lai, M., Guez, A.,
  Lanctot, M., Sifre, L., Kumaran, D., Graepel, T., Lillicrap, T.~P., Simonyan,
  K., and Hassabis, D.
\newblock {Mastering Chess and Shogi by Self-Play with a General Reinforcement
  Learning Algorithm.}
\newblock \emph{CoRR}, abs/1712.01815, 2017.

\bibitem[Sinha et~al.(2022)Sinha, Mandlekar, and Garg]{s4rl2021}
Sinha, S., Mandlekar, A., and Garg, A.
\newblock S4rl: Surprisingly simple self-supervision for offline reinforcement
  learning in robotics.
\newblock In \emph{Conference on Robot Learning}, pp.\  907--917. PMLR, 2022.

\bibitem[Sriperumbudur et~al.(2009)Sriperumbudur, Fukumizu, Gretton, Scholkopf,
  and Lanckriet]{ipm2009}
Sriperumbudur, B.~K., Fukumizu, K., Gretton, A., Scholkopf, B., and Lanckriet,
  G. R.~G.
\newblock {On integral probability metrics, $\phi$-divergences and binary
  classification}.
\newblock \emph{{arXiv: Information Theory}}, 2009.

\bibitem[Sutton \& Barto(2018)Sutton and Barto]{rlintro2018}
Sutton, R.~S. and Barto, A.~G.
\newblock \emph{Reinforcement learning: An introduction}.
\newblock MIT press, 2018.

\bibitem[Swazinna et~al.(2021)Swazinna, Udluft, and Runkler]{moose2021}
Swazinna, P., Udluft, S., and Runkler, T.~A.
\newblock {Overcoming Model Bias for Robust Offline Deep Reinforcement
  Learning}.
\newblock \emph{{Eng. Appl. Artif. Intell.}}, 104:\penalty0 104366, 2021.

\bibitem[Tang et~al.(2020)Tang, Feng, Li, Zhou, and Liu]{tang2019doubly}
Tang, Z., Feng, Y., Li, L., Zhou, D., and Liu, Q.
\newblock Doubly robust bias reduction in infinite horizon off-policy
  estimation.
\newblock In \emph{International Conference on Learning Representations}, 2020.

\bibitem[Urp\'i et~al.(2021)Urp\'i, Curi, and Krause]{oraac2021}
Urp\'i, N.~A., Curi, S., and Krause, A.
\newblock {Risk-Averse Offline Reinforcement Learning}.
\newblock \emph{ArXiv}, abs/2102.05371, 2021.

\bibitem[Wang et~al.(2020)Wang, Novikov, Zolna, Merel, Springenberg, Reed,
  Shahriari, Siegel, Gulcehre, Heess, and de~Freitas]{crr2020}
Wang, Z., Novikov, A., Zolna, K., Merel, J.~S., Springenberg, J.~T., Reed,
  S.~E., Shahriari, B., Siegel, N., Gulcehre, C., Heess, N., and de~Freitas, N.
\newblock {Critic Regularized Regression}.
\newblock In Larochelle, H., Ranzato, M., Hadsell, R., Balcan, M.~F., and Lin,
  H. (eds.), \emph{{Advances in Neural Information Processing Systems}},
  volume~33, pp.\  7768--7778. Curran Associates, Inc., 2020.

\bibitem[Wasserman(2006)]{nonparamstat2006}
Wasserman, L.
\newblock \emph{{All of Nonparametric Statistics}}.
\newblock Springer, 2006.

\bibitem[White(2016)]{samplegenerative2016}
White, T.
\newblock Sampling generative networks.
\newblock \emph{arXiv preprint arXiv:1609.04468}, 2016.

\bibitem[Wu et~al.(2019)Wu, Tucker, and Nachum]{brac2019}
Wu, Y., Tucker, G., and Nachum, O.
\newblock Behavior regularized offline reinforcement learning.
\newblock \emph{arXiv preprint arXiv:1911.11361}, 2019.

\bibitem[Wu et~al.(2021)Wu, Zhai, Srivastava, Susskind, Zhang, Salakhutdinov,
  and Goh]{uwac2021}
Wu, Y., Zhai, S., Srivastava, N., Susskind, J., Zhang, J., Salakhutdinov, R.,
  and Goh, H.
\newblock {Uncertainty Weighted Actor-Critic for Offline Reinforcement
  Learning}.
\newblock In \emph{{International Conference on Machine Learning}}, 2021.

\bibitem[Yang et~al.(2022)Yang, Wang, Zheng, Feng, and Zhou]{jointmatching2022}
Yang, S., Wang, Z., Zheng, H., Feng, Y., and Zhou, M.
\newblock A regularized implicit policy for offline reinforcement learning.
\newblock \emph{arXiv preprint arXiv:2202.09673}, 2022.

\bibitem[Yu et~al.(2020)Yu, Thomas, Yu, Ermon, Zou, Levine, Finn, and
  Ma]{mopo2020}
Yu, T., Thomas, G., Yu, L., Ermon, S., Zou, J.~Y., Levine, S., Finn, C., and
  Ma, T.
\newblock {MOPO: Model-based Offline Policy Optimization}.
\newblock \emph{{ArXiv}}, abs/2005.13239, 2020.

\bibitem[Yu et~al.(2021)Yu, Kumar, Rafailov, Rajeswaran, Levine, and
  Finn]{combo2021}
Yu, T., Kumar, A., Rafailov, R., Rajeswaran, A., Levine, S., and Finn, C.
\newblock {COMBO: Conservative Offline Model-Based Policy Optimization}.
\newblock \emph{ArXiv}, abs/2102.08363, 2021.

\bibitem[Yue et~al.(2020)Yue, Wang, and Zhou]{idac2020}
Yue, Y., Wang, Z., and Zhou, M.
\newblock {Implicit Distributional Reinforcement Learning}.
\newblock In Larochelle, H., Ranzato, M., Hadsell, R., Balcan, M., and Lin, H.
  (eds.), \emph{{Advances in Neural Information Processing Systems 33: Annual
  Conference on Neural Information Processing Systems 2020, NeurIPS 2020,
  December 6-12, 2020, virtual}}, 2020.

\bibitem[Yurtsever et~al.(2020)Yurtsever, Lambert, Carballo, and
  Takeda]{autodrivesurvey2020}
Yurtsever, E., Lambert, J., Carballo, A., and Takeda, K.
\newblock {A Survey of Autonomous Driving: Common Practices and Emerging
  Technologies}.
\newblock \emph{IEEE Access}, 8:\penalty0 58443--58469, 2020.

\bibitem[Zhang et~al.(2020)Zhang, Dai, Li, and Schuurmans]{gendice2020}
Zhang, R., Dai, B., Li, L., and Schuurmans, D.
\newblock Gendice: Generalized offline estimation of stationary values.
\newblock In \emph{International Conference on Learning Representations}, 2020.

\bibitem[Zhang et~al.(2021)Zhang, Fan, Zheng, Tanwisuth, and
  Zhou]{zhang2021alignment}
Zhang, S., Fan, X., Zheng, H., Tanwisuth, K., and Zhou, M.
\newblock Alignment attention by matching key and query distributions.
\newblock \emph{Advances in Neural Information Processing Systems}, 34, 2021.

\bibitem[Zhang et~al.(2022)Zhang, Gong, Liu, He, Chen, and
  Zhou]{zhang2022allsh}
Zhang, S., Gong, C., Liu, X., He, P., Chen, W., and Zhou, M.
\newblock Allsh: Active learning guided by local sensitivity and hardness.
\newblock \emph{arXiv preprint arXiv:2205.04980}, 2022.

\bibitem[Zheng \& Zhou(2021)Zheng and Zhou]{act2021}
Zheng, H. and Zhou, M.
\newblock {Exploiting Chain Rule and Bayes' Theorem to Compare Probability
  Distributions}.
\newblock \emph{{Advances in Neural Information Processing Systems}}, 34, 2021.

\end{thebibliography}
